\documentclass[runningheads]{llncs}
\usepackage[preprint,year=2026]{eccv}
\usepackage{eccvabbrv}
\usepackage{graphicx}
\usepackage{booktabs}
\usepackage[accsupp]{axessibility}  
\usepackage{multirow}
\usepackage{pifont}
\usepackage{makecell}
\usepackage{caption}
\usepackage[table]{xcolor}
\definecolor{colBlueLight}{RGB}{225,238,255}
\definecolor{colBlue}{RGB}{180,210,255}
\definecolor{colGreenLight}{RGB}{225,245,230}
\definecolor{colGreen}{RGB}{185,230,195}
\usepackage{subcaption}
\usepackage{algorithm}
\usepackage{algpseudocode}
\usepackage{placeins}
\usepackage[pagebackref,breaklinks,colorlinks,citecolor=eccvblue]{hyperref}
\usepackage{orcidlink}

\begin{document}
\title{FD$^2$: A Dedicated Framework for Fine-Grained Dataset Distillation}

\titlerunning{Fine-Grained Dataset Distillation}

\author{
Hongxu Ma\inst{1}$^\dagger$ \and
Guang Li\inst{2}$^\dagger$\textsuperscript{*} \and
Shijie Wang\inst{3} \and
Dongzhan Zhou\inst{4} \and
Baoli Sun\inst{5} \and
Takahiro Ogawa\inst{2} \and
Miki Haseyama\inst{2} \and
Zhihui Wang\inst{5}\textsuperscript{*}
}


\authorrunning{H.~Ma and G.~Li et al.}

\institute{
Zhejiang University \and
Hokkaido University \and
The University of Queensland \and
Shanghai AI Laboratory \and
Dalian University of Technology \\
$^\dagger$ Equal Contribution, 
\textsuperscript{*}Corresponding Authors: \\ \texttt{guang@lmd.ist.hokudai.ac.jp}, \texttt{zhihuiwang@dlut.edu.cn}
}

\maketitle

\begin{abstract}
Dataset distillation (DD) compresses a large training set into a small synthetic set, reducing storage and training cost, and has shown strong results on general benchmarks. Decoupled DD further improves efficiency by splitting the pipeline into pretraining, sample distillation, and soft-label generation.
However, existing decoupled methods largely rely on coarse class-label supervision and optimize samples within each class in a nearly identical manner. On fine-grained datasets, this often yields distilled samples that \textbf{(i)} retain large intra-class variation with subtle inter-class differences and \textbf{(ii)} become overly similar within the same class, limiting localized discriminative cues and hurting recognition.
To solve the above-mentioned problems, we propose \textbf{FD$^{2}$}, a dedicated framework for \textbf{F}ine-grained \textbf{D}ataset \textbf{D}istillation. FD$^{2}$ localizes discriminative regions and constructs fine-grained representations for distillation. During pretraining, counterfactual attention learning aggregates discriminative representations to update class prototypes. During distillation, a fine-grained characteristic constraint aligns each sample with its class prototype while repelling others, and a similarity constraint diversifies attention across same-class samples.
Experiments on multiple fine-grained and general datasets show that FD$^{2}$ integrates seamlessly with decoupled DD and improves performance in most settings, indicating strong transferability. Code is available at \href{https://github.com/Guang000/FD2}{https://github.com/Guang000/FD2}.

\end{abstract}

\begin{figure}[t]
  \centering

  \begin{subfigure}[t]{\linewidth}
    \centering
    \includegraphics[width=\linewidth]{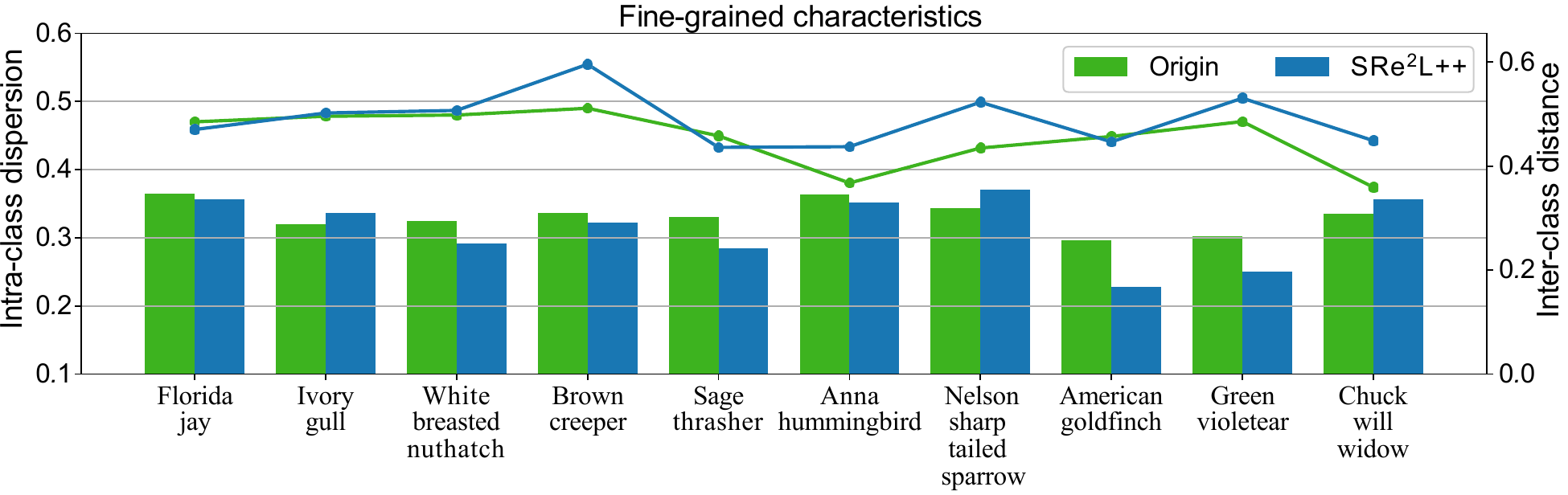}
    \caption{}
    \label{fig:fg_full_vs_sre2l}
  \end{subfigure}

  \vspace{0pt}

  \begin{subfigure}[t]{\linewidth}
    \centering
    \includegraphics[width=\linewidth]{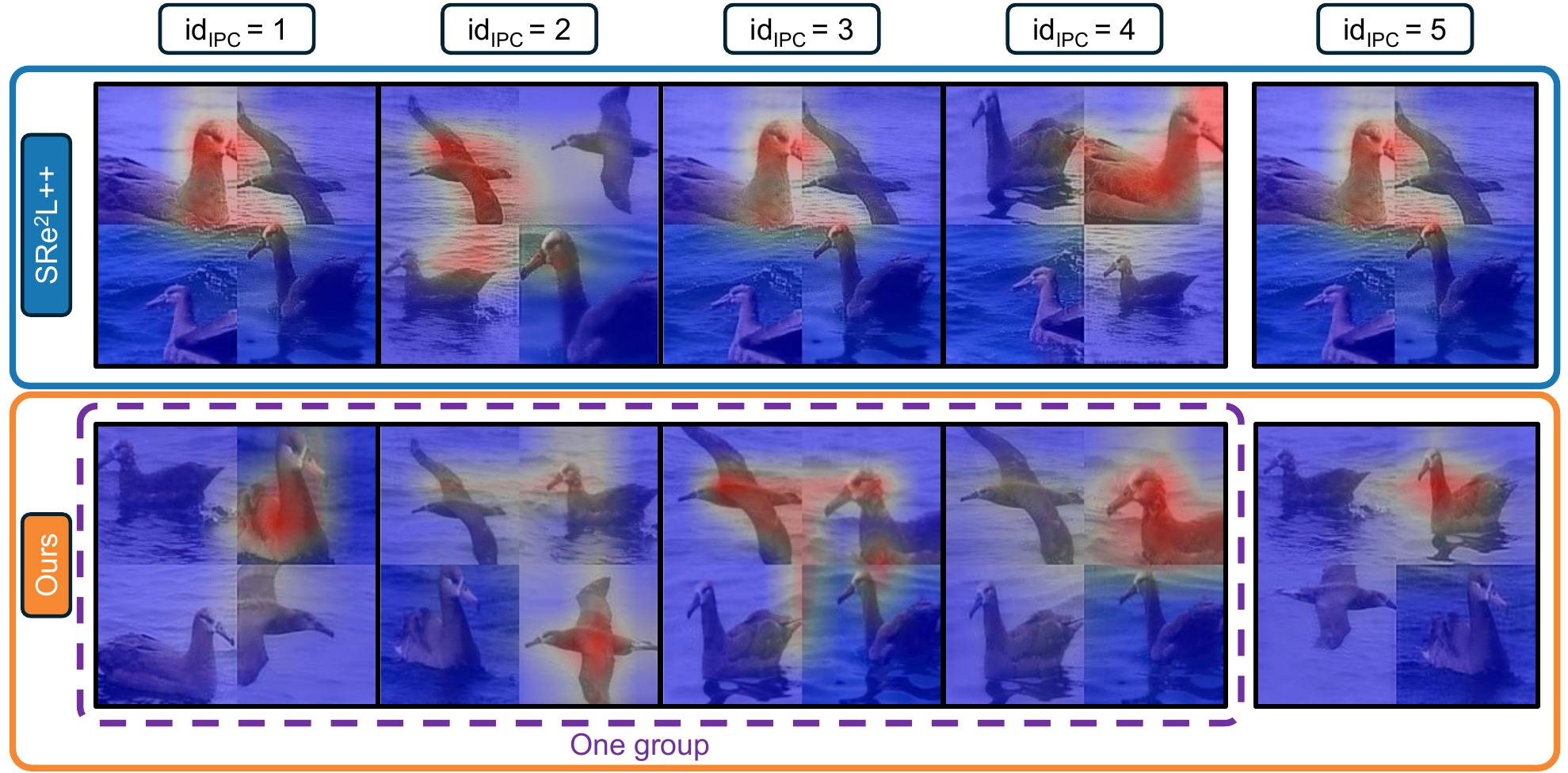}
    \caption{}
    \label{fig:attn_sre2l_vs_ours}
  \end{subfigure}

  \caption{(a) On CUB-200-2011, 10 classes are randomly sampled to compare the original training set and the SRe$^{2}$L++ distilled set, where intra-class dispersion (bars) and inter-class distance (line) are reported. (b) Attention heatmaps of \emph{Black-footed Albatross} in CUB-200-2011, comparing the attended regions of distilled samples produced by SRe$^{2}$L++ and FD$^{2}$.}
  \label{fig:fg_and_attn}
\end{figure}

\section{Introduction}
Recently, the rapid growth of large-scale datasets has driven substantial progress in artificial intelligence, but it also brings high computational cost and long training time \cite{ImageNet,ViT,ResNet}. Dataset distillation (DD) addresses this challenge by learning a compact synthetic dataset that preserves the training utility of the original data, such that a model trained on the distilled set achieves performance close to training on the full dataset while using far fewer resources \cite{DC-BENCH,li2020soft,DD,DC,li2022awesome, li2022compressed}.
Existing DD methods can be broadly grouped into matching-based approaches, \eg, gradient matching \cite{DC,DSA}, distribution matching \cite{DM, CAFE, li2025hdd}, trajectory matching \cite{MTT, TESLA, li2024dataset, li2024iadd}, decoupled methods \cite{SRe2L, CVDD, FADRM}, and generative methods \cite{ITGAN, li2024generative, Minimax, su2024diffusion, li2025diff, zou2025dataset}. Most of these techniques are evaluated on general-purpose benchmarks, including ImageNette and ImageWoof \cite{ImageNet}, where they have shown promising performance.
Among these, decoupled methods \cite{SRe2L, CVDD, FADRM} split dataset distillation into three stages: model pretraining, sample distillation, and soft-label generation, which significantly improves efficiency while preserving competitive accuracy on general benchmarks.
However, when directly applied to fine-grained datasets (\eg, CUB-200-2011), existing decoupled methods suffer from two limitations: \textbf{(i)} the distilled set preserves unfavorable fine-grained characteristics, and \textbf{(ii)} distilled samples within the same class become overly similar.

For \textbf{(i)}, \cref{fig:fg_full_vs_sre2l} visualizes 10 randomly selected CUB-200-2011 classes, where the original data exhibit large intra-class variation and subtle inter-class differences, which we term the fine-grained characteristic. This characteristic degrades the model’s ability to extract discriminative cues, thereby limiting fine-grained recognition performance. A representative decoupled method, SRe$^2$L++ \cite{CVDD}, largely inherits this unfavorable characteristic in the distilled set, and in some classes, the characteristic is even worse than that of the full dataset, making it difficult for the student model to learn discriminative representations. We attribute this to the fact that decoupled DD relies primarily on coarse class-label supervision during distillation: it enforces class-level semantic consistency but does not explicitly encourage intra-class compactness or inter-class separability at the fine-grained level.
For \textbf{(ii)}, the attention heatmaps in \cref{fig:attn_sre2l_vs_ours} show that SRe$^2$L++ attends to highly consistent regions across same-class distilled samples, which reduces coverage of discriminative parts and weakens the diversity of discriminative cues available for the student model. This behavior is induced by sample-wise iterative synthesis under nearly identical optimization across iterations, which drives the same-class samples to converge to similar solutions.
Together, these issues limit the quality of fine-grained distilled data and consequently degrade the student model’s performance on fine-grained recognition.

To address these issues, we propose \textbf{FD$^{2}$}, a dedicated framework for \textbf{F}ine-grained \textbf{D}ataset \textbf{D}istillation that augments decoupled DD with fine-grained supervision. FD$^{2}$ leverages Counterfactual Attention Learning (CAL) to extract attention maps that localize discriminative regions and to construct class prototypes summarizing these representations. These attention-based signals provide fine-grained supervision for the subsequent distillation process.
During pretraining, CAL is used to obtain attention maps and update class prototypes that capture discriminative cues. During distillation, the attention maps are used to localize discriminative regions and are fused with backbone features to form fine-grained representations. Based on these representations, we introduce two complementary constraints. A \emph{fine-grained characteristic constraint} pulls each distilled sample toward its target-class prototype while pushing it away from other-class prototypes, improving intra-class compactness and inter-class separability. A \emph{similarity constraint} encourages diversity among same-class distilled samples by separating their attention distributions, enabling coverage of different discriminative regions. 
Both constraints are incorporated into existing distillation objectives as plug-in terms, introducing minimal overhead while preserving the original decoupled distillation pipeline. Experiments on multiple fine-grained and general datasets demonstrate that FD$^{2}$ consistently improves representative decoupled distillation methods.

The main contributions of this paper are summarized as follows:
\begin{itemize}
    \item We propose FD$^{2}$, a dedicated framework for fine-grained dataset distillation that augments decoupled DD with fine-grained supervision. The proposed method improves the structure of the distilled dataset by simultaneously enhancing inter-class separability and promoting within-class diversity.
    \item We introduce a fine-grained characteristic constraint to improve intra-class compactness and inter-class separability, and a similarity constraint to diversify attention across same-class distilled samples, thereby providing richer localized discriminative cues.
    \item Extensive experiments on multiple fine-grained and general datasets demonstrate that FD$^{2}$ can be integrated into representative decoupled methods without modification of their workflow and achieves consistent improvements in most settings, with particularly strong gains on fine-grained datasets.
\end{itemize}

\section{Related Work}
\label{sec:related}

\paragraph{Dataset Distillation.}
Dataset distillation (DD) aims to synthesize a compact set of training samples from a large dataset, substantially reducing storage and training cost while retaining performance close to training on the full data. Existing DD methods can be broadly categorized into gradient matching \cite{DC,DSA,DCC,DREAM}, distribution matching \cite{DM,CAFE,HaBa,DataDAM,GUARD,li2025davdd}, trajectory matching \cite{MTT,TESLA,APM,DATM}, decoupled distillation \cite{SRe2L,FADRM,GVBSM,EDC,CDA,LPLD}, and generative distillation \cite{ITGAN,HPD,Minimax,D4M,wu2025dc3,ye2025igds,li2025diffusion,cai2026evlf}.

\paragraph{Decoupled Dataset Distillation.}
SRe$^2$L \cite{SRe2L} pioneers decoupled distillation by separating model pretraining from sample distillation. SRe$^2$L++ \cite{CVDD} improves robustness via stronger augmentation and batch-specific soft labels. FADRM \cite{FADRM} enriches pixel-space information with multi-scale residual connections (with FADRM+ extending to multi-model distillation). G-VBSM \cite{GVBSM} enhances generalization across multiple models, while CDA \cite{CDA} stabilizes optimization through a curriculum strategy. LPLD \cite{LPLD} revisits the necessity of large-scale soft labels and proposes lighter-weight alternatives. Despite their effectiveness on general-purpose benchmarks, these methods are not designed with fine-grained datasets in mind and do not explicitly address fine-grained separability or within-class diversity in the distilled set.

\paragraph{Fine-grained Image Recognition.}
Fine-grained image recognition (FGIR) often relies on localized discriminative regions, motivating methods that mine and aggregate part-level cues. Cross-X \cite{Cross-X} exploits cross-layer interaction with cross-category constraints. PMG \cite{PMG} learns multi-granularity cues via progressive training. DP-Net \cite{DP-Net} injects learnable positional cues and dynamically aligns them with visual content. CSC-Net \cite{CSC-Net} improves localization by enforcing class-specific semantic coherency. CAL \cite{CAL} introduces counterfactual attention learning to enhance attention quality and guide models toward key local cues.

\section{Approach}
\label{sec:approach}

\subsection{Preliminaries}
\label{sec:preliminaries}
Dataset distillation (DD) aims to synthesize a compact training set that preserves the training utility of a large labeled dataset. Let $\mathcal{T}=\{(x_i,y_i)\}_{i=1}^{N}$ denote the original training set. The objective is to distill a substantially smaller set $\mathcal{D}=\{(\tilde{x}_j,\tilde{y}_j)\}_{j=1}^{M}$ with $M\ll N$, such that models trained on $\mathcal{D}$ achieve performance comparable to those trained on $\mathcal{T}$. This objective can be expressed by minimizing the discrepancy between the losses obtained under the two training regimes:
\begin{equation}
\label{eq:dd_gap}
\sup_{(x,y)\sim \mathcal{T}}
\left| \mathcal{L}\!\left(f_{\theta_{\mathcal{T}}}(x),y\right)
-\mathcal{L}\!\left(f_{\theta_{\mathcal{D}}}(x),y\right)\right|
\le \delta,
\end{equation}
where $\mathcal{L}(\cdot,\cdot)$ denotes the task loss (\eg, cross-entropy), and $\theta_{\mathcal{T}}$ and $\theta_{\mathcal{D}}$ represent the parameters obtained by training on $\mathcal{T}$ and $\mathcal{D}$, respectively. Accordingly, dataset distillation can be formulated as the following optimization problem:
\begin{equation}
\label{eq:dd_obj}
\arg\min_{\mathcal{D},\,|\mathcal{D}|}
\sup_{(x,y)\sim \mathcal{T}}
\left| \mathcal{L}\!\left(f_{\theta_{\mathcal{T}}}(x),y\right)
-\mathcal{L}\!\left(f_{\theta_{\mathcal{D}}}(x),y\right)\right|.
\end{equation}
Under a given distillation budget, the distilled set $\mathcal{D}$ maintains an equal number of images for each class.

\subsection{Overview of FD$^{2}$}
\begin{figure}[t]
  \centering
  \includegraphics[width=\linewidth]{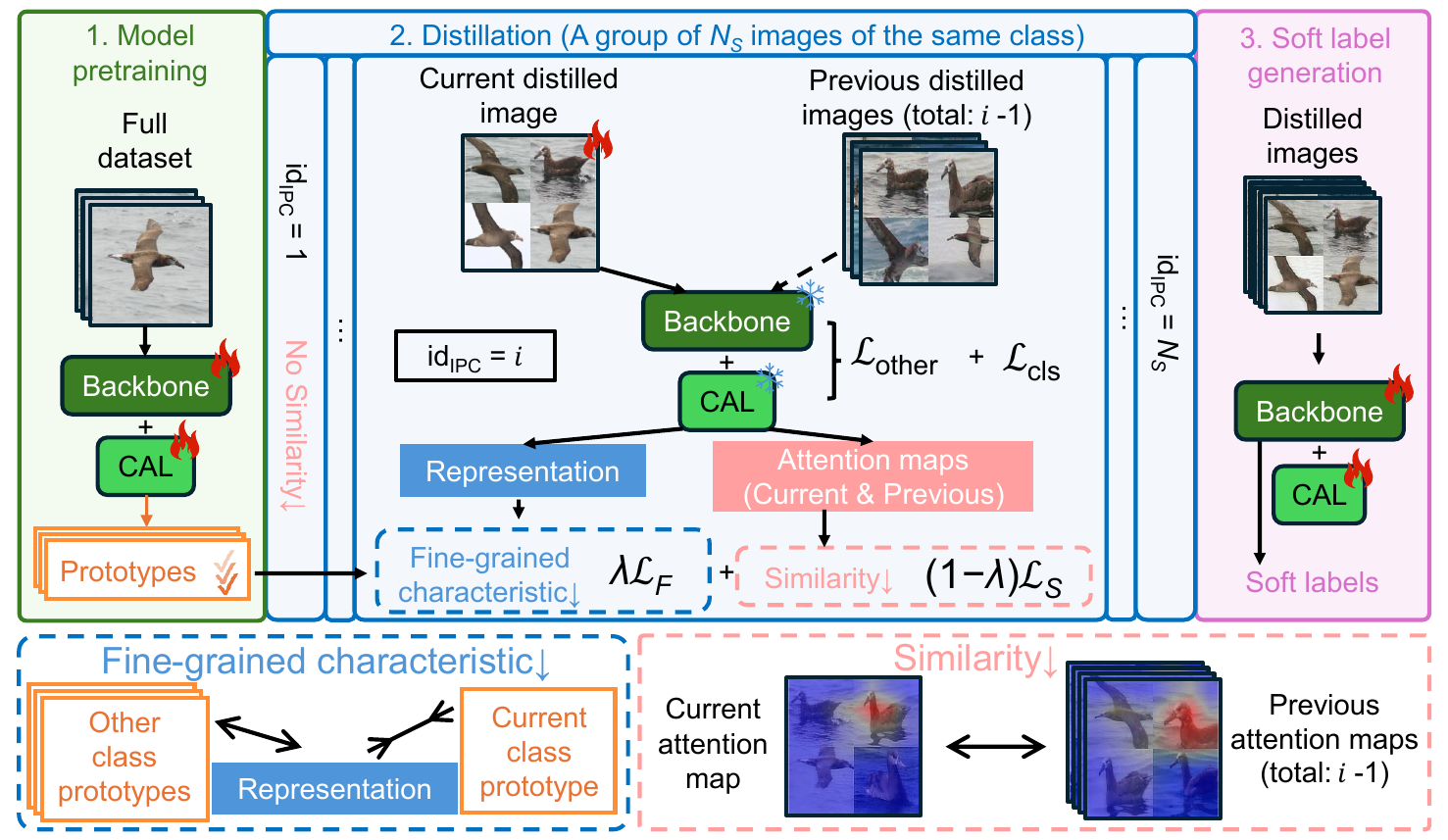}
  \caption{Overview of FD$^{2}$. 1. We pretrain a Backbone+CAL teacher and maintain class prototypes online. 2. We distill images in the same-class groups of size $N_S$, adding a fine-grained characteristic constraint (prototype alignment or separation) and a similarity constraint (diverse attention), together with class-supervision from both branches. 3. We generate soft labels using the backbone branch.}
  \label{fig:overview}
\end{figure}

The overall framework of FD$^{2}$ is illustrated in \cref{fig:overview}. FD$^{2}$ follows the decoupled distillation paradigm and can be integrated into existing decoupled methods as an additive module. We first pretrain a Backbone+CAL teacher by jointly optimizing the backbone and CAL classifiers while maintaining fine-grained class prototypes online.
During distillation, distilled images are optimized in a sample-wise manner and organized into the same-class groups of size $N_S$. For each current distilled sample, its feature representation and attention map are obtained from the Backbone+CAL teacher, and the original distillation objective is augmented with two constraints. The fine-grained characteristic constraint pulls the sample toward the prototype of the target class and pushes it away from prototypes of other classes, improving intra-class compactness and inter-class separability. The similarity constraint encourages within-class diversity by increasing the discrepancy between the attention map of the current sample and those of previously generated samples of the same class, thereby promoting coverage of different discriminative regions.
Finally, during the soft-label generation stage, soft labels are produced by the backbone branch. Since downstream evaluation trains a standard backbone student, this design helps reduce bias caused by architectural mismatch.

\subsection{Counterfactual Attention Learning (CAL)}
\label{sec:cal}
To address the limitations of previous methods on fine-grained datasets, we introduce Counterfactual Attention Learning (CAL) \cite{CAL}, a simple and effective approach that provides fine-grained class prototypes and precise attention maps. CAL leverages counterfactual intervention to construct a factual--counterfactual contrastive signal, enabling more reliable attention localization and discriminative representation learning under weak supervision. 

Given an input image $x$ with label $y$, a backbone feature extractor $f(\cdot)$ produces the final-stage feature map $F$, and an attention predictor $g(\cdot)$ generates $M$ attention maps $A=\{A_m\}_{m=1}^{M}$. CAL then applies an attention-weighted aggregation operator $\Phi(\cdot)$ (\eg, bilinear attention pooling) to obtain the factual representation $z$:
\begin{equation}
F=f(x),\qquad A=g(F),\qquad z=\Phi(F,A).
\label{eq:cal_factual}
\end{equation}

To explicitly quantify the contribution of attention, CAL performs a counterfactual intervention by replacing $A$ with an alternative (or perturbed) attention set $\bar A$ while keeping $F$ fixed. This operation produces a counterfactual representation $\hat z$, and the difference between factual and counterfactual predictions forms an effect prediction:
\begin{equation}
\hat z=\Phi(F,\bar A),\qquad p_{\mathrm{raw}}=Wz,\qquad p_{\mathrm{eff}}=p_{\mathrm{raw}}-W\hat z,
\label{eq:cal_counterfactual}
\end{equation}
where $W$ denotes the linear classifier of CAL, $p_{\mathrm{raw}}$ represents the factual logit, and $p_{\mathrm{eff}}$ measures the gain induced by the factual attention relative to the counterfactual attention. Intuitively, if attention focuses on truly discriminative regions, replacing it with $\bar A$ weakens class evidence, which increases the discriminability of $p_{\mathrm{eff}}$ and is therefore encouraged during training.

Meanwhile, CAL maintains a feature center for each class as a class prototype. Let $c_y$ denote the prototype of class $y$. The prototype is updated online using a momentum rule, and a center regularizer encourages the factual representation $z$ to approach the corresponding class prototype:
\begin{equation}
c_y \leftarrow (1-\mu)c_y+\mu\,\mathrm{Norm}(z),\qquad 
\mathcal{L}_{\mathrm{center}}=\big\|z-\mathrm{Norm}(c_y)\big\|_2^2,
\label{eq:cal_center}
\end{equation}
where $\mathrm{Norm}(\cdot)$ denotes feature normalization and $\mu\in(0,1)$ denotes the momentum coefficient.

Subsequently, the pretraining objective of CAL is defined as:
\begin{equation}
\mathcal{L}_{\mathrm{CAL}}
=\mathcal{L}_{\mathrm{ce}}(p_{\mathrm{raw}},y)
+\mathcal{L}_{\mathrm{ce}}(p_{\mathrm{eff}},y)
+\eta\mathcal{L}_{\mathrm{center}},
\label{eq:cal_loss}
\end{equation}
where $\mathcal{L}_{\mathrm{ce}}(\cdot,\cdot)$ denotes the cross-entropy loss and $\eta$ controls the strength of the center regularization. This objective enables CAL to identify discriminative regions more reliably and to learn discriminative representations during training, which improves classification accuracy.

During the pretraining, distillation, and soft-label generation stages, the teacher adopts the Backbone+CAL architecture. In the post-evaluation stage, following the standard protocol of prior methods, a plain Backbone student model is used. Notably, during pretraining, the backbone classifier and the CAL classifier share the same backbone features and are jointly optimized. This design is motivated by the use of a backbone student model in the post-evaluation stage. If soft labels generated by the CAL branch are used directly, instability may be introduced into the probability outputs of the student model. Formally, the backbone branch predicts:
\begin{equation}
p_{\mathrm{bb}}=W_{\mathrm{bb}}\cdot\mathrm{GAP}(F),
\label{eq:bb_pred}
\end{equation}
where $W_{\mathrm{bb}}$ denotes the backbone classifier and $\mathrm{GAP}(\cdot)$ denotes global average pooling. The overall pretraining loss is defined as a weighted combination:
\begin{equation}
\mathcal{L}_{\mathrm{pre}}=(1-\alpha)\,\mathcal{L}_{\mathrm{ce}}(p_{\mathrm{bb}},y)\;+\;\alpha\,\mathcal{L}_{\mathrm{CAL}},
\label{eq:pre_loss}
\end{equation}
where $\alpha\in[0,1]$ controls the relative contribution of the two branches. Joint optimization ensures that the backbone teacher model preserves the ability to generate stable soft labels for student training, while still benefiting from the discriminative representations and prototype aggregation provided by CAL.

\subsection{Constraints}
\label{sec:constraints}
Suppressing unfavorable fine-grained characteristics and reducing similarity among distilled samples are two key strategies for fine-grained dataset distillation. Therefore, two specialized constraints are introduced.

\paragraph{Fine-grained Characteristic Constraint.}
Suppressing unfavorable fine-grained characteristics in distilled samples reduces the difficulty for the student model to learn discriminative representations. A straightforward strategy aligns the representation of a sample with the prototype of its class while increasing the distance to the prototypes of other classes:
\begin{equation}
\mathcal{L}_{F}(\tilde{x}_{y,i})=
\beta\,\ell_2(z_{y,i},c_y)+
(1-\beta)\left(1-\mathbb{E}_{k\neq y}\big[\ell_2(z_{y,i},c_k)\big]\right),
\qquad \beta\in[0,1],
\label{eq:LF}
\end{equation}
where $\tilde{x}_{y,i}$ denotes the $i$-th distilled image of class $y$ in a group, $\beta$ is a weighting hyperparameter, $\ell_2(u,v)=\|u-v\|_2/(\|u\|_2+\|v\|_2+\varepsilon)$ is a symmetrically normalized Euclidean metric, $z_{y,i}$ denotes the representation of the sample, and $c_y$ and $c_k$ denote the prototypes of the current class and other classes, respectively. According to Eq.~\eqref{eq:LF}, the intra-class compactness and inter-class separability of the distilled images are enhanced.

\paragraph{Similarity Constraint.}
To suppress similarity among distilled samples and encourage diverse discriminative regions, the distance between the attention map of the current sample and the attention maps of previously distilled samples is maximized in the attention space:
\begin{equation}
\mathcal{L}_{S}(\tilde{x}_{y,i})=
1-\mathbb{E}_{j<i}[\ell_2\big(A_{y,i},A_{y,j})\big],\qquad 1<i\le N_S,
\label{eq:LS}
\end{equation}
and $\mathcal{L}_{S}$ is omitted when $i=1$. Here, $A_{y,i}$ denotes the attention map of the current sample, $A_{y,j}$ denotes the attention maps of previous samples of the same class, and $N_S$ denotes the number of distilled images in the group that satisfy this constraint.

Since pretraining adopts a dual-classifier architecture, distillation also supervises class signals with both classifiers to leverage their discriminative capability:
\begin{equation}
\mathcal{L}_{\mathrm{cls}}(\tilde{x}_{y,i})=
(1-\alpha)\,\mathcal{L}_{\mathrm{ce}}\!\left(p^{y,i}_{\mathrm{bb}},y\right)
+\alpha\,\mathcal{L}_{\mathrm{ce}}\!\left(p^{y,i}_{\mathrm{cal}},y\right),
\qquad \alpha\in[0,1].
\label{eq:Lcls}
\end{equation}
Here, $\alpha$ is a weighting hyperparameter, and $p^{y,i}_\mathrm{bb}$ and $p^{y,i}_\mathrm{cal}$ denote the logits produced by the backbone and CAL branches, respectively. Let $\mathcal{L}_{\mathrm{other}}$ denote the original loss of the underlying distillation method. The final objective for optimizing $\tilde{x}_{y,i}$ is
\begin{equation}
\mathcal{L}_{\mathrm{total}}=\mathcal{L}_{\mathrm{other}}+\mathcal{L}_{\mathrm{cls}}+\lambda\mathcal{L}_{F}+(1-\lambda)\mathcal{L}_{S},
\label{eq:Ltotal}
\end{equation}
where $\lambda$ controls the relative contribution of the two proposed constraints. The training process is summarized in \cref{alg:dfdd}. The proposed constraints are jointly optimized with the original objective, which suppresses undesirable fine-grained characteristics and similarity without degrading the effectiveness of the underlying method. A theoretical analysis of these constraints is provided in the supplementary material.

\begin{algorithm}[t]
\caption{FD$^{2}$}
\label{alg:dfdd}
\begin{algorithmic}[1]
\Require Class index $y$, number of classes $K$, group size $N_S$
\Require Pretrained teacher $\phi(\cdot)$ (Backbone+CAL), prototypes $\{c_k\}_{k=1}^{K}$
\Require Weights $\alpha,\beta,\lambda$, original loss $\mathcal{L}_{\mathrm{other}}$
\For{$i=1$ to $N_{S}$}
    \State Initialize distilled image $\tilde{x}_{y,i}$
    \If{$i>1$}
        \For{$j=1$ to $i-1$}
            \State $(\cdot,\cdot,\cdot, A_{y,j}) \gets \phi(\tilde{x}_{y,j})$ \Comment{Attention maps of previous same-class samples}
        \EndFor
    \EndIf
    \For{each step}
        \State $(p^{y,i}_\mathrm{bb},\,p^{y,i}_\mathrm{cal},\,z_{y,i},\,A_{y,i}) \gets \phi(\tilde{x}_{y,i})$ \Comment{Logits, representation and attention map}
        \State Compute $\mathcal{L}_{F}$ by Eq.~\eqref{eq:LF}
        \If{$i>1$}
            \State Compute $\mathcal{L}_{S}$ by Eq.~\eqref{eq:LS} using $\{A_{y,j}\}_{j=1}^{i-1}$ and $A_{y,i}$
        \Else
            \State $\mathcal{L}_{S} \gets 0$
        \EndIf
        \State Compute $\mathcal{L}_{\mathrm{cls}}$ by Eq.~\eqref{eq:Lcls}
        \State $\mathcal{L}_{\mathrm{total}} \gets \mathcal{L}_{\mathrm{other}}+\mathcal{L}_{\mathrm{cls}}+\lambda\mathcal{L}_{F}+(1-\lambda)\mathcal{L}_{S}$
        \State Update $\tilde{x}_{y,i}$ by gradient descent on $\mathcal{L}_{\mathrm{total}}$
    \EndFor
\EndFor
\end{algorithmic}
\end{algorithm}

\section{Experiments}
\subsection{Experimental Settings}

\paragraph{Datasets.}
Evaluation of FD$^{2}$ is conducted on three fine-grained datasets, CUB-200-2011~\cite{CUB-200-2011}, FGVC-Aircraft~\cite{FGVC-Aircraft}, and Stanford Cars~\cite{StanfordCars}, together with two general datasets, ImageNette and ImageWoof, which are subsets of ImageNet-1K~\cite{ImageNet}. All input images are resized to $224\times224$.

\paragraph{Comparative Methods.}
Comparison of FD$^{2}$ is performed with three state-of-the-art dataset distillation baselines under the same experimental environment on a single H200 GPU. Since FD$^{2}$ is designed for the decoupled distillation paradigm, SRe$^{2}$L++~\cite{CVDD} is selected as the primary baseline. This method extends SRe$^{2}$L~\cite{SRe2L} through stronger data augmentation and batch-specific soft labels. To evaluate plug-and-play transferability, FADRM+~\cite{FADRM} is also included. FADRM+ is an ensemble-based decoupled method that improves both efficiency and performance through multi-scale residual connections. In addition, RDED~\cite{RDED} is adopted as a representative approach that constructs distilled sets by cropping real images, which can be interpreted as exploiting localized regions during the distillation process.

\smallskip
\smallskip

\noindent In cross-architecture generalization experiments and ablation studies, CUB-200-2011 is used as the default dataset unless otherwise specified. The supplementary material provides additional ablation studies, further comparisons with other advanced methods on fine-grained datasets, and efficiency analysis.

\begin{table}[t]
\centering
\caption{Top-1 accuracy is reported in the post-evaluation stage to compare FD$^{2}$ with three state-of-the-art methods. For fine-grained datasets, where the number of samples per class is limited, the value of IPC is set to 1, 3, and 5. Following the standard protocol of each method, post-evaluation is conducted for 300 epochs for RDED, which constructs distilled sets by cropping real images, and for 400 epochs for the decoupled baselines SRe$^{2}$L++ and FADRM+. Evaluation of FD$^{2}$ is performed by integrating it into SRe$^{2}$L++ and FADRM+ (denoted by the subscript ``$_{\mathrm{FD^{2}}}$'') while keeping the same post-evaluation settings as the corresponding baseline methods. SRe$^{2}$L++ and its FD$^{2}$ version follow the standard same-architecture evaluation protocol, while FADRM+ and its FD$^{2}$ version use multi-teacher for DD and single student for post-evaluation.}
\label{tab:fg}
\scriptsize
\setlength{\tabcolsep}{2pt}
\renewcommand{\arraystretch}{0.6}

\resizebox{\linewidth}{!}{%
\begin{tabular}{
  c c c c
  >{\columncolor{colBlueLight}}c
  >{\columncolor{colBlue}}c
  >{\columncolor{colGreenLight}}c
  >{\columncolor{colGreen}}c
}
\toprule
Dataset & Student & IPC & RDED & SRe$^2$L++ & SRe$^2$L++$_{\mathrm{FD^{2}}}$ & FADRM+ & FADRM+$_{\mathrm{FD^{2}}}$ \\
\midrule
\multirow{6}{*}[-1.2ex]{CUB-200-2011}
& \multirow{3}{*}{ResNet18} & 1 & 38.3 & 53.4 & 56.4\,($\uparrow 3.0$) & 54.8 & 55.0\,($\uparrow 0.2$) \\
&                           & 3 & 52.6 & 60.0 & 64.9\,($\uparrow 4.9$) & 64.0 & 64.6\,($\uparrow 0.6$) \\
&                           & 5 & 63.9 & 63.5 & 67.0\,($\uparrow 3.5$) & 66.4 & 67.5\,($\uparrow 1.1$) \\
\cmidrule(lr){2-8}
& \multirow{3}{*}{ResNet50} & 1 & 33.4 & 61.1 & 70.1\,($\uparrow 9.0$) & 66.2 & 66.5 ($\uparrow 0.3$) \\
&                           & 3 & 49.0 & 65.1 & 73.7\,($\uparrow 8.6$) & 70.4 & 70.6\,($\uparrow 0.2$) \\
&                           & 5 & 58.6 & 68.1 & 75.5\,($\uparrow 7.4$) & 71.5 & 72.0\,($\uparrow 0.5$) \\
\midrule
\multirow{6}{*}[-1.2ex]{FGVC-Aircraft}
& \multirow{3}{*}{ResNet18} & 1 & 22.1 & 52.6 & 58.2\,($\uparrow 5.6$)  & 55.0 & 60.5\,($\uparrow 5.5$) \\
&                           & 3 & 36.4 & 66.6 & 76.1\,($\uparrow 9.5$)  & 72.9 & 75.1\,($\uparrow 2.2$) \\
&                           & 5 & 38.6 & 68.3 & 80.0\,($\uparrow 11.7$) & 74.0 & 77.6\,($\uparrow 3.6$) \\
\cmidrule(lr){2-8}
& \multirow{3}{*}{ResNet50} & 1 & 12.8 & 59.3 & 67.8\,($\uparrow 8.5$) & 70.8 & 73.8\,($\uparrow 3.0$) \\
&                           & 3 & 34.8 & 68.0 & 76.7\,($\uparrow 8.7$) & 76.1 & 78.9\,($\uparrow 2.8$) \\
&                           & 5 & 45.2 & 72.2 & 79.0\,($\uparrow 6.8$) & 78.6 & 81.0\,($\uparrow 2.4$) \\
\midrule
\multirow{6}{*}[-1.2ex]{Stanford Cars}
& \multirow{3}{*}{ResNet18} & 1 & 33.0 & 52.4 & 64.5\,($\uparrow 12.1$) & 60.3 & 74.1\,($\uparrow 13.8$) \\
&                           & 3 & 69.4 & 68.2 & 75.2\,($\uparrow 7.0$)  & 75.0 & 84.8\,($\uparrow 9.8$) \\
&                           & 5 & 76.1 & 70.9 & 81.4\,($\uparrow 10.5$) & 77.7 & 86.6\,($\uparrow 8.9$) \\
\cmidrule(lr){2-8}
& \multirow{3}{*}{ResNet50} & 1 & 12.8 & 65.6 & 80.7\,($\uparrow 15.1$) & 73.9 & 85.3\,($\uparrow 11.4$) \\
&                           & 3 & 70.2 & 76.6 & 86.5\,($\uparrow 9.9$)  & 78.9 & 87.8\,($\uparrow 8.9$) \\
&                           & 5 & 75.9 & 78.0 & 88.3\,($\uparrow 10.3$) & 82.2 & 89.0\,($\uparrow 6.8$) \\
\bottomrule
\end{tabular}}
\end{table}

\subsection{Main Results}

\paragraph{Fine-grained Datasets.} As shown in \cref{tab:fg}, FD$^{2}$ improves the performance of decoupled methods in most settings when compared with SRe$^2$L++ and FADRM+. This result indicates that the mitigation of unfavorable fine-grained characteristics and excessive similarity among samples can improve the quality of distilled data and facilitate the learning of discriminative representations by the student model. Notably, the improvements are typically more pronounced when IPC$=1$, with particularly clear gains on the Stanford Cars dataset. For example, with ResNet18, SRe$^2$L++$_{\mathrm{FD^{2}}}$ exceeds SRe$^2$L++ by 12.1\%, and FADRM+$_{\mathrm{FD^{2}}}$ exceeds FADRM+ by 13.8\%. This observation suggests that FD$^{2}$ can still provide effective discriminative cues even when only a single distilled sample is available for each class. In contrast, RDED exhibits the lowest overall performance, which indicates that coarse cropping of real images does not reliably capture key discriminative regions. 

\paragraph{General datasets.} To verify the applicability of FD$^{2}$ beyond fine-grained datasets, we further evaluate it on ImageNette and ImageWoof. Since SRe$^2$L++ does not report results on ImageWoof, we use RDED and FADRM+ as baselines.
As shown in \cref{tab:general}, on ImageWoof (which can be viewed as a 10-class fine-grained subset of ImageNet), FADRM+$_{\mathrm{FD^{2}}}$ achieves clear improvements in most settings, supporting the effectiveness of FD$^{2}$.
On ImageNette, the effect of FD$^{2}$ depends on the student architecture: with ResNet50, FADRM+$_{\mathrm{FD^{2}}}$ yields clear improvements across all IPC settings, while with the smaller ResNet18, the gains are more limited and become noticeable mainly at IPC$=50$. This suggests that the additional fine-grained cues introduced by FD$^{2}$ are more effectively utilized by higher-capacity students.
RDED consistently performs the worst across datasets and IPC settings, indicating that a cropping-based manner provides limited useful information on general datasets.
Overall, FD$^{2}$ improves performance in most settings, with reduced gains mainly observed for small-capacity students on simpler datasets.

\begin{table}[t]
\centering
\scriptsize
\setlength{\tabcolsep}{3pt}
\renewcommand{\arraystretch}{0.6}
\caption{Top-1 accuracy in the post-evaluation stage is reported to compare FD$^{2}$ with two baselines. On general datasets, IPC is set to 1, 10, and 50. Following the standard protocols, post-evaluation is conducted for 300 epochs for both RDED and FADRM+. FD$^{2}$ is evaluated by integrating it into FADRM+ (denoted by the subscript ``$_{\mathrm{FD^{2}}}$'') under the same post-evaluation setting.}
\label{tab:general}
\begin{tabular}{
  c c
  c >{\columncolor{colGreenLight}}c >{\columncolor{colGreen}}c
  c >{\columncolor{colGreenLight}}c >{\columncolor{colGreen}}c
}
\toprule
\multirow{2}{*}{Dataset} & \multirow{2}{*}{IPC} &
\multicolumn{3}{c}{ResNet18} & \multicolumn{3}{c}{ResNet50} \\
\cmidrule(lr){3-5}\cmidrule(lr){6-8}
& & RDED & FADRM+ & FADRM+$_{\mathrm{FD^{2}}}$ & RDED & FADRM+ & FADRM+$_{\mathrm{FD^{2}}}$ \\
\midrule

\multirow{3}{*}{ImageNette}
& 1  & 35.8 & 39.2 & 38.6\,($\downarrow 0.6$)  & 27.0 & 31.9 & 39.0\,($\uparrow 7.1$) \\
& 10 & 61.4 & 69.0 & 69.5\,($\uparrow 0.5$) & 55.0 & 68.1 & 71.4\,($\uparrow 3.3$)  \\
& 50 & 80.4 & 84.6 & 86.7\,($\uparrow 2.1$)    & 81.8 & 85.4 & 92.3\,($\uparrow 6.9$)  \\
\midrule

\multirow{3}{*}{ImageWoof}
& 1  & 20.8 & 22.8 & 22.1\,($\downarrow 0.7$)  & 17.8 & 19.9 & 30.3\,($\uparrow 10.4$) \\
& 10 & 38.5 & 57.3 & 60.7\,($\uparrow 3.4$)    & 35.2 & 54.1 & 64.6\,($\uparrow 10.5$) \\
& 50 & 68.5 & 72.6 & 80.0\,($\uparrow 7.4$)    & 67.0 & 71.7 & 80.2\,($\uparrow 8.5$)  \\
\bottomrule
\end{tabular}
\end{table}

\subsection{Analysis}
\paragraph{Fine-grained Characteristic.}
As shown in \cref{fig:tsne_intra_disp}, we visualize distilled samples on CUB-200-2011 using t-SNE \cite{t-SNE} in the feature space, where 10 classes are randomly sampled with 5 distilled images per class. Compared with SRe$^{2}$L++, same-class samples distilled with FD$^{2}$ exhibit higher intra-class compactness. Meanwhile, \cref{fig:inter_class_dist} shows larger nearest-neighbor center distances between classes under FD$^{2}$, indicating improved inter-class separability. These results demonstrate that the proposed fine-grained characteristic constraint reduces intra-class variance while enhancing inter-class separability.

\begin{figure}[t]
  \centering
  \begin{subfigure}[t]{0.34\linewidth}
    \centering
    \includegraphics[width=\linewidth]{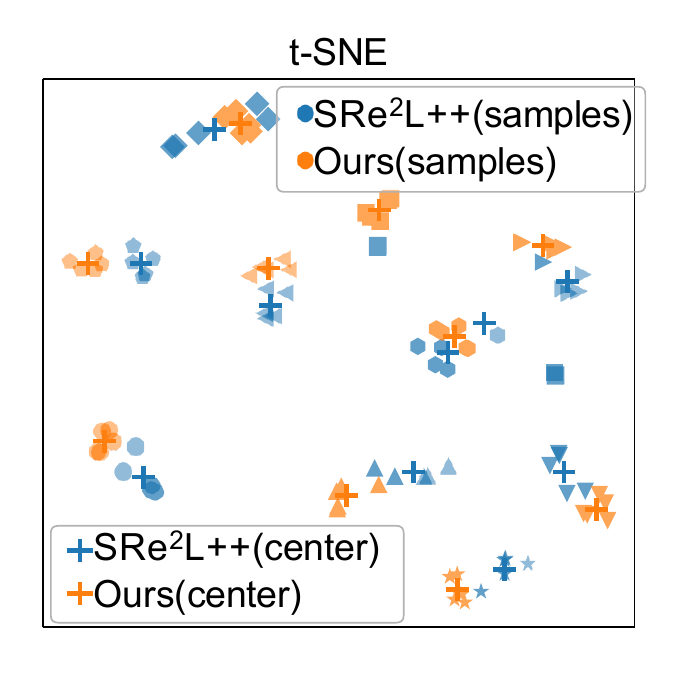}
    \caption{}
    \label{fig:tsne_intra_disp}
  \end{subfigure}\hspace{0.02\linewidth}%
  \begin{subfigure}[t]{0.64\linewidth}
    \centering
    \includegraphics[width=\linewidth]{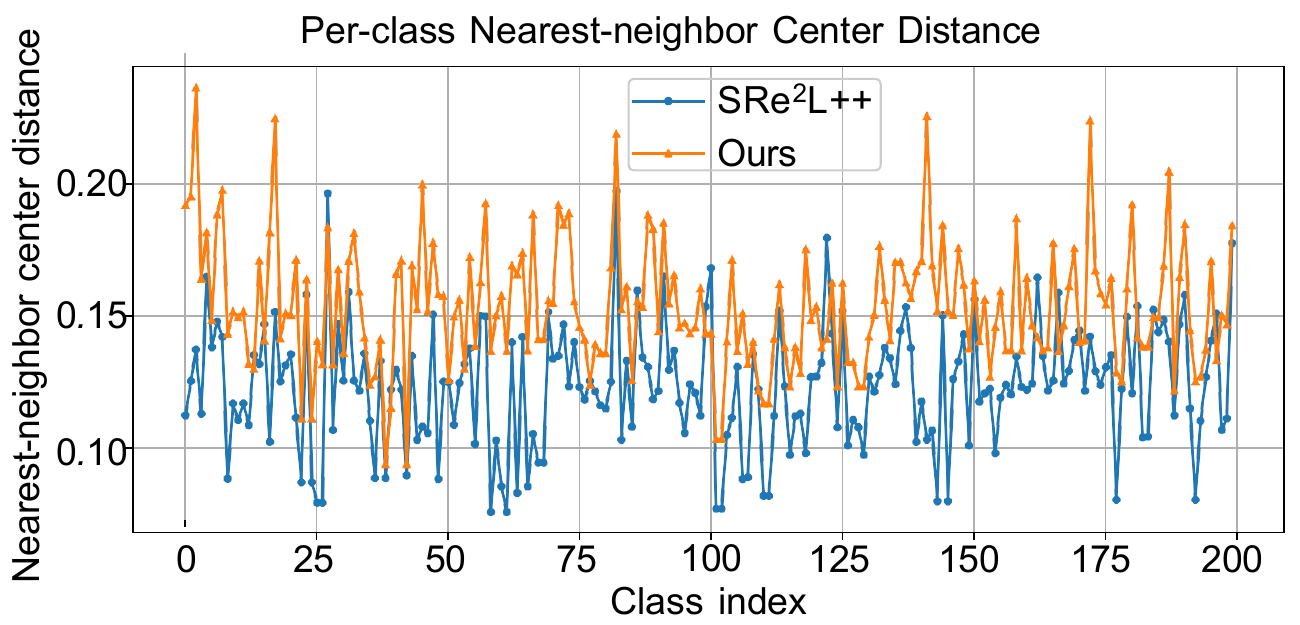}
    \caption{}
    \label{fig:inter_class_dist}
  \end{subfigure}
  \caption{(a) t-SNE feature distribution. (b) Nearest-neighbor center distance of each class for distilled images on CUB-200-2011.}
  \label{fig:tsne_and_inter}
\end{figure}

\paragraph{Similarity.}
As illustrated in \cref{fig:attn_sre2l_vs_ours}, attention heatmaps of the black-footed albatross class in CUB-200-2011 show that SRe$^{2}$L++ produces highly homogeneous distilled samples within the same class, where the attended regions are nearly identical across samples. When the similarity constraint is applied within a group of $\mathrm{id}_{\mathrm{IPC}}{=}1\text{-}4$, FD$^{2}$ encourages more diverse attention distributions, enriching the discriminative regions in the distilled samples. Although RDED introduces some local information through cropping, this coarse strategy does not reliably capture key discriminative cues. As a result, the performance of RDED is generally lower than that of FD$^{2}$, as shown in \cref{tab:fg} and \cref{tab:general}.

\paragraph{Transferability.}
FD$^{2}$ functions as an add-on module that can be integrated into different decoupled distillation methods without modifying the original training pipeline. As shown in \cref{tab:fg}, integrating FD$^{2}$ into SRe$^2$L++ yields clear accuracy improvements across datasets and IPC settings. When integrated into FADRM+, improvements are also observed on most datasets, although the gains on CUB-200-2011 remain relatively limited, while larger improvements appear on Stanford Cars. These results indicate that FD$^{2}$ consistently enhances the performance of different decoupled methods and demonstrates strong transferability.

\subsection{Cross-Architecture Generalization}

\begin{table}[t]
  \centering
  \scriptsize
  \setlength{\tabcolsep}{3pt}
  \caption{Top-1 accuracy for cross-architecture generalization at IPC=3.}
  \label{tab:cag}
  \begin{tabular}{c
    >{\columncolor{colBlueLight}}c
    >{\columncolor{colBlue}}c
    >{\columncolor{colGreenLight}}c
    >{\columncolor{colGreen}}c}
    \toprule
    Student &
    SRe$^2$L++ & SRe$^2$L++$_{\mathrm{FD^{2}}}$ & FADRM+ & FADRM+$_{\mathrm{FD^{2}}}$ \\
    \midrule
    ShuffleNetV2 \cite{ShuffleNet} & 38.5 & 47.6\,($\uparrow 9.1$) & 46.4 & 46.7\,($\uparrow 0.3$) \\
    MobileNetV2 \cite{MobileNetV2} & 57.1 & 62.9\,($\uparrow 5.8$) & 64.9 & 65.8\,($\uparrow 0.9$) \\
    DenseNet121 \cite{DenseNet}    & 62.0 & 65.8\,($\uparrow 3.8$) & 68.6 & 68.1\,($\downarrow 0.5$) \\
    ResNet18 \cite{ResNet}         & 60.0 & 64.9\,($\uparrow 4.9$) & 64.0 & 64.6\,($\uparrow 0.6$) \\
    ResNet50 \cite{ResNet}         & 64.2 & 67.7\,($\uparrow 3.5$) & 70.4 & 70.6\,($\uparrow 0.2$) \\
    ConvNeXt-Tiny \cite{ConvNeXt}   & 28.9 & 30.4\,($\uparrow 1.5$) & 29.1 & 35.1\,($\uparrow 0.2$) \\
    \bottomrule
  \end{tabular}
\end{table}

Cross-architecture generalization is an important criterion for evaluating the quality of distilled datasets, as it reflects transferability across different students and practical applicability.
To examine this property, we compare SRe$^2$L++, FADRM+, and their FD$^{2}$-augmented variants. For SRe$^2$L++ and SRe$^2$L++$_{\mathrm{FD}^{2}}$, we use ResNet18 as the teacher and evaluate with different students; we also include ResNet18 as the student for ease of comparison. For FADRM+ and FADRM+$_{\mathrm{FD}^{2}}$, we distill with multiple teachers and use one student for evaluation.
As shown in \cref{tab:cag}, integrating FD$^{2}$ improves performance for most architectures and methods. On DenseNet121 \cite{DenseNet}, FADRM+ performs slightly better than FADRM+$_{\mathrm{FD^{2}}}$, while SRe$^2$L++$_{\mathrm{FD^{2}}}$ consistently outperforms SRe$^2$L++. These results indicate that FD$^{2}$ provides strong cross-architecture generalization and stable performance across diverse evaluation settings.

\subsection{Ablation Study}
To reduce computational cost and runtime, FD$^{2}$ is integrated into SRe$^{2}$L++, and all ablation studies are conducted using ResNet18 + CAL.

\noindent
\begin{minipage}[t]{0.48\textwidth}\vspace{0pt}
  \centering
  \includegraphics[width=\linewidth]{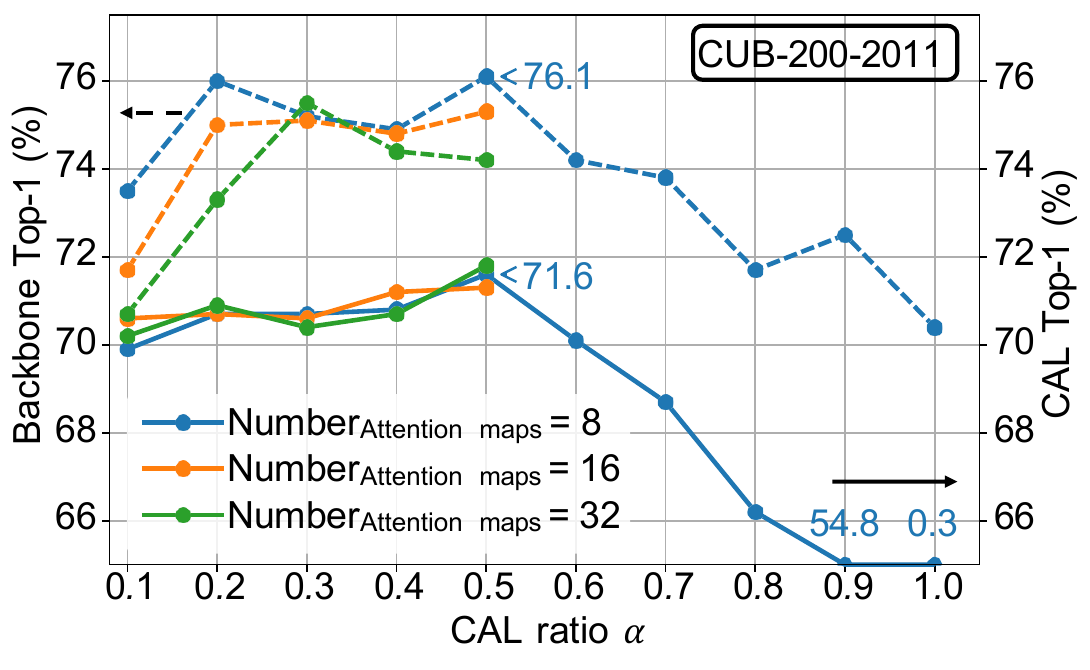}
  \vspace{-8pt}
  \captionsetup{type=figure,hypcap=false,skip=0pt}
  \captionof{figure}{Top-1 accuracy under different settings during pretraining.}
  \label{fig:top1_squeeze}
\end{minipage}\hspace{0.01\textwidth}
\begin{minipage}[t]{0.50\textwidth}\vspace{0pt}
  \centering
  \small
\captionsetup{type=table,hypcap=false,skip=2pt,justification=justified,singlelinecheck=false}
  \captionof{table}{Accuracy at IPC$=3$ on CUB-200-2011 and FGVC-Aircraft for distilled datasets obtained with different CAL ratios during distillation.}
  \label{tab:top1_recover}
  \setlength{\tabcolsep}{4pt}
  \renewcommand{\arraystretch}{1.0}
  \begin{tabular}{c c c}
    \toprule
    $\alpha$ & CUB-200-2011 & FGVC-Aircraft \\
    \midrule
    0.1 & 62.8 & 75.1 \\
    0.3 & 63.3 & \textbf{75.9} \\
    0.5 & \textbf{63.4} & 74.8 \\
    0.7 & 62.1 & 75.6 \\
    0.9 & 62.4 & 75.2 \\
    \bottomrule
  \end{tabular}
\end{minipage}

\paragraph{Impact of the Cal Ratio $\alpha$ in the Model Pretraining and Distillation.}
Since the backbone classifier and the CAL classifier are jointly optimized, we first perform an ablation study in the pretraining stage to determine an appropriate CAL ratio, as shown in \cref{fig:top1_squeeze}. With the number of attention maps fixed to 8, different CAL ratios are evaluated. When the CAL ratio exceeds 0.5, the accuracy of both the backbone branch and the CAL branch decreases noticeably. Therefore, in experiments with 16 and 32 attention maps, the CAL ratio is restricted to 0.1--0.5. Considering all configurations, the pretrained model with 8 attention maps and a CAL ratio of 0.5 is selected for the distillation stage.

During distillation, the effect of the CAL ratio on post-evaluation performance is further examined with IPC$=3$. As shown in \cref{tab:top1_recover}, on CUB-200-2011 the highest Top-1 accuracy (63.4\%) is obtained at CAL ratio $=0.5$. On FGVC-Aircraft, the best accuracy (75.9\%) is achieved at CAL ratio $=0.3$, which is also optimal in the pretraining stage. These results suggest that using the same CAL ratio in both stages helps maintain consistent feature distributions and attention strengths, thereby reducing stage mismatch. Therefore, adopting the same CAL ratio across stages provides a stable configuration.

\noindent
\begin{minipage}[t]{0.49\textwidth}\vspace{0pt}
\centering
\small
\captionsetup{type=table,hypcap=false,skip=1pt,justification=justified,singlelinecheck=false,width=\linewidth}
\captionof{table}{Impact of $\beta$ in the fine-grained characteristic constraint.}
\label{tab:beta}
\setlength{\tabcolsep}{2pt}
\renewcommand{\arraystretch}{0.8}
\begin{tabular}{ccc}
\toprule
Method & $\beta$ & Acc \\
\midrule
\multirow{5}{*}{SRe$^2$L++$_{\mathrm{FD}^{\scriptscriptstyle 2}}$}
 & 0.00 & 63.5 \\
 & 0.25 & 63.7 \\
 & 0.50 & \textbf{64.8} \\
 & 0.75 & 64.7 \\
 & 1.00 & 64.1 \\
\bottomrule
\end{tabular}
\end{minipage}\hspace{0.02\textwidth}%
\begin{minipage}[t]{0.49\textwidth}\vspace{0pt}
\centering
\small
\captionsetup{type=table,hypcap=false,skip=1pt,justification=justified,singlelinecheck=false,width=\linewidth}
\captionof{table}{Impact of the group size $N_S$ in the similarity constraint.}
\label{tab:NS}
\setlength{\tabcolsep}{2pt}
\renewcommand{\arraystretch}{0.95}
\begin{tabular}{ccc}
\toprule
Method & $N_S$ & Acc \\
\midrule
\multirow{4}{*}{SRe$^2$L++$_{\mathrm{FD}^{\scriptscriptstyle 2}}$}
  & 2 & 64.1 \\
  & 3 & 64.9 \\
  & 4 & \textbf{66.1} \\
  & 5 & 64.6 \\
\bottomrule
\end{tabular}
\end{minipage}

\paragraph{Impact of $\beta$ in the Fine-grained Characteristic Constraint.}
To study the relative weighting between alignment with same-class prototypes and repulsion from other-class prototypes, we enable the fine-grained characteristic constraint alone and ablate $\beta$ at IPC$=3$. As shown in \cref{tab:beta}, the best post-evaluation accuracy is achieved at $\beta=0.5$. This result indicates that a balanced weighting improves both intra-class compactness and inter-class separability, thereby improving the quality of samples. Therefore, $\beta=0.5$ is used in subsequent experiments.

\paragraph{Impact of the Group Size $N_S$ in the Similarity Constraint.}
To determine an appropriate group size $N_S$, we ablate $N_S$ with only the similarity constraint enabled and IPC$=5$. As shown in \cref{tab:NS}, $N_S=4$ achieves the best post-evaluation accuracy (66.1\%). When $N_S$ increases to 5, the available discriminative cues become limited, and additional samples tend to focus on repeated regions. With only a few distilled samples, this repetition may cause the student model to overfit to these regions, which reduces performance (64.6\%). Therefore, $N_S=4$ is selected. For larger IPC, multiple $N_S=4$ groups can be processed in parallel with multiple processes/GPUs. The distillation times for different $N_S$ settings are reported in the supplementary material.

\noindent
\begin{minipage}[t]{0.49\textwidth}\vspace{0pt}
\centering
\small
\captionsetup{type=table,hypcap=false,skip=1pt,justification=justified,singlelinecheck=false,width=\linewidth}
\captionof{table}{Ablation study of different constraint combinations.}
\label{tab:contraints}
\setlength{\tabcolsep}{1.2pt}
\renewcommand{\arraystretch}{0.95}
\resizebox{0.63\linewidth}{!}{%
\begin{tabular}{c c c c}
\toprule
Method & $\mathcal{L}_F$ & $\mathcal{L}_S$ & Acc \\
\midrule
\multirow{4}{*}{SRe$^2$L++$_{\mathrm{FD}^{\scriptscriptstyle 2}}$}
  & --        & --        & 63.4 \\
  & \ding{51} & --        & 64.8 \\
  & --        & \ding{51} & 64.6 \\
  & \ding{51} & \ding{51} & \textbf{64.9} \\
\bottomrule
\end{tabular}}
\end{minipage}\hspace{0.02\textwidth}%
\begin{minipage}[t]{0.49\textwidth}\vspace{0pt}
\centering
\small
\captionsetup{type=table,hypcap=false,skip=1pt,justification=justified,singlelinecheck=false,width=\linewidth}
\captionof{table}{The impact of the relative weight $\lambda$ between the proposed constraints.}
\label{tab:mixing_weight}
\setlength{\tabcolsep}{3pt}
\renewcommand{\arraystretch}{0.95}
\begin{tabular}{c c c}
\toprule
Method & $\lambda$ & Top-1 Acc \\
\midrule
\multirow{4}{*}{SRe$^2$L++$_{\mathrm{FD}^{\scriptscriptstyle 2}}$}
  & 0.2 & 65.4 \\
  & 0.4 & 65.0 \\
  & 0.6 & 65.7 \\
  & 0.8 & \textbf{67.0} \\
\bottomrule
\end{tabular}
\end{minipage}

\paragraph{Effectiveness of the Proposed Constraints.}
To evaluate the contribution of the proposed constraints, we compare different constraint combinations with IPC$=3$. As shown in \cref{tab:contraints}, both constraints improve performance over the no-constraint baseline (63.4\%). The fine-grained characteristic constraint provides a slightly larger gain (64.8\%) than the similarity constraint (64.6\%). When both constraints are enabled, the highest accuracy (64.9\%) is achieved, which confirms their complementary effects.

\paragraph{Effect of the Relative Weight $\lambda$ between the Proposed Constraints.} To determine the trade-off between the proposed constraints, we ablate the relative weight $\lambda$ under IPC$=5$. As shown in \cref{tab:mixing_weight}, $\lambda=0.8$ achieves higher post-evaluation Top-1 accuracy, while smaller values (0.2 and 0.4) degrade performance. This suggests that a larger weight for the fine-grained characteristic constraint improves distilled sample quality. Thus, we set $\lambda=0.8$ in subsequent experiments.

\subsection{Visualization}

\paragraph{Attention Visualization under Similarity Constraint.}
As shown in \cref{fig:vis_similarity}, we compare attention visualizations and mean pairwise cosine similarity (MPCS$\downarrow$) with and without $\mathcal{L}_{S}$. Without $\mathcal{L}_{S}$, same-class samples focus on repeated regions and show higher MPCS; with $\mathcal{L}_{S}$, they cover richer discriminative regions and show lower MPCS. Thus, although $\mathcal{L}_{S}$ brings limited accuracy gains in \cref{tab:contraints}, it improves sample diversity by reducing redundant attention and enriching discriminative cues.

\paragraph{Visualization Comparison of Distilled Samples across Methods.}
\cref{fig:vis_distilled} presents distilled samples from several CUB-200-2011 classes produced by different methods. Samples generated by SRe$^2$L++ and FADRM+ often exhibit noticeable Gaussian noise. In contrast, integrating FD$^{2}$ produces samples with clearer local structures and richer texture details. This observation indicates that the proposed method preserves fine-grained local cues more effectively and improves inter-class separability, which enhances the overall quality of the distilled dataset.

\noindent
\begin{minipage}[t]{0.48\textwidth}\vspace{0pt}
  \centering
  \includegraphics[width=\linewidth]{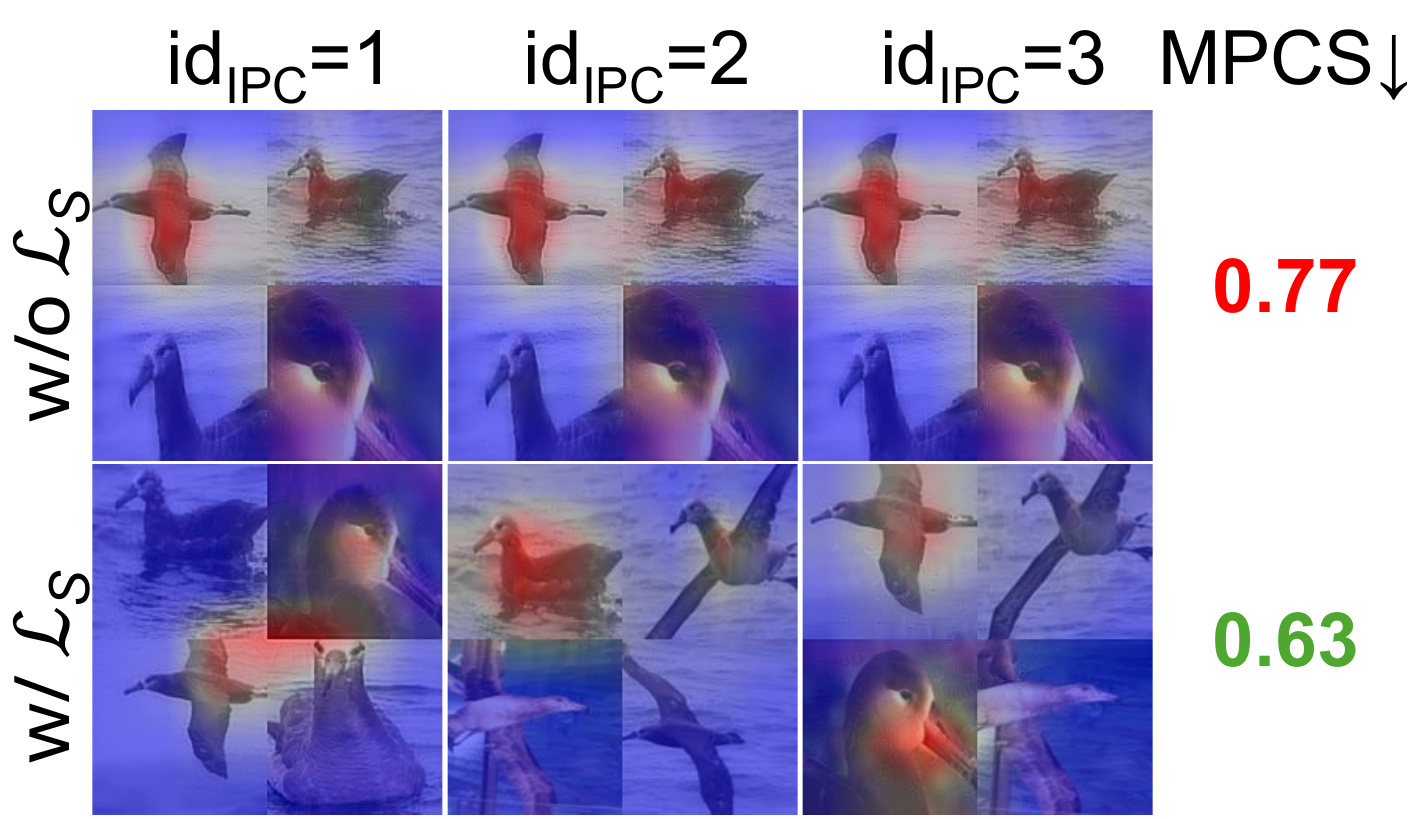}
  \captionsetup{type=figure,hypcap=false,skip=0pt}
  \captionof{figure}{Attention heatmaps and mean pairwise cosine similarity (MPCS$\downarrow$) for the $\mathrm{id}_{\mathrm{IPC}}=1\text{-}3$ distilled samples of the black footed albatross class in CUB-200-2011, with and without the similarity constraint under $N_S=4$.}

  \label{fig:vis_similarity}
\end{minipage}\hspace{0.01\textwidth} 
\begin{minipage}[t]{0.48\textwidth}\vspace{0pt}
  \centering
  \includegraphics[width=\linewidth]{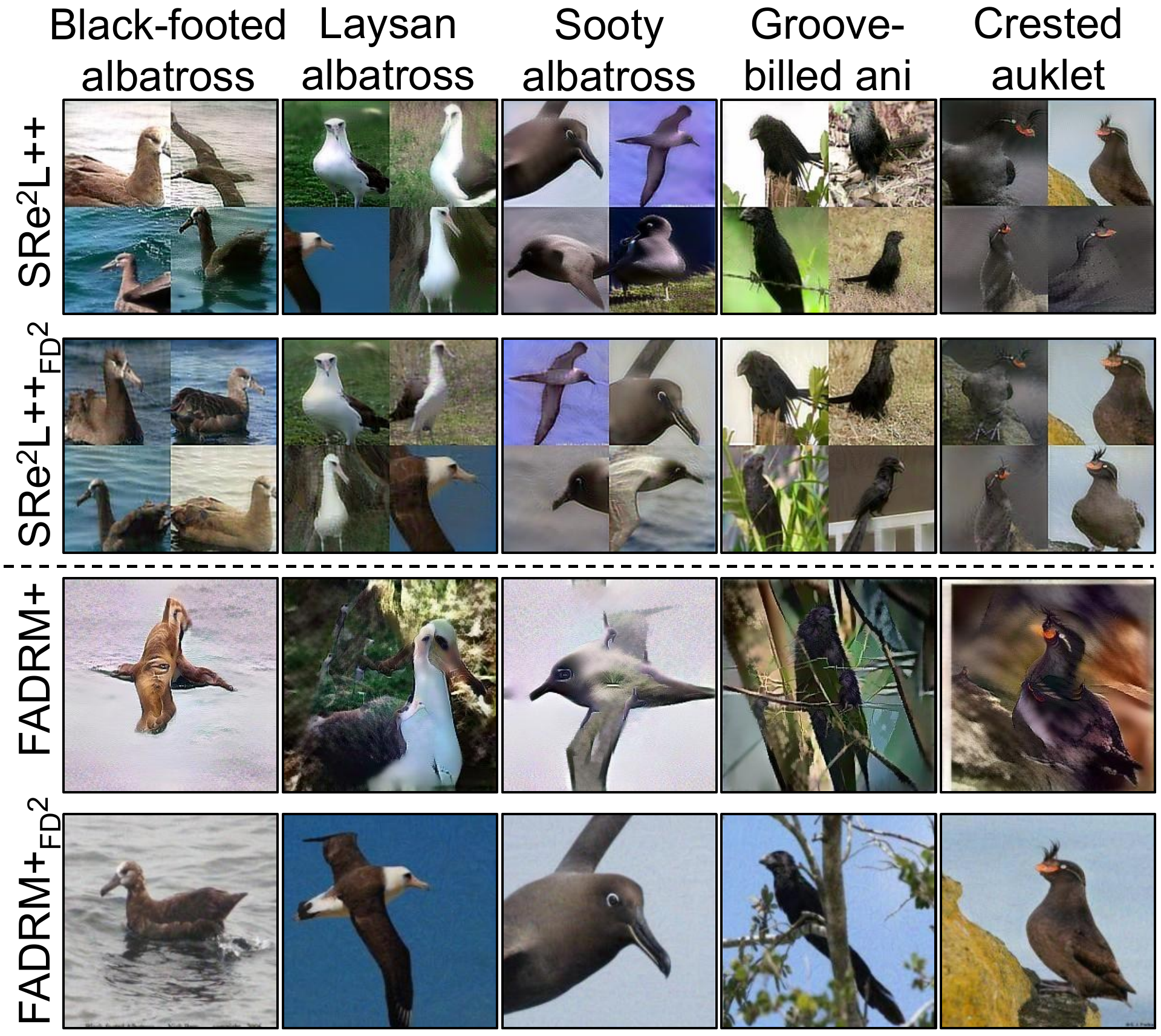}
  \captionsetup{type=figure,hypcap=false,skip=0pt} 
  \captionof{figure}{Visual comparison of distilled samples generated by different methods on five randomly selected classes.}
  \label{fig:vis_distilled}
\end{minipage}

\section{Conclusion}
We propose FD$^{2}$, a dedicated framework designed for fine-grained dataset distillation.
By jointly optimizing a fine-grained characteristic constraint and a similarity constraint together with the original distillation objective, FD$^{2}$ mitigates unfavorable fine-grained characteristics and within-class homogenization while preserving the decoupled distillation pipeline.
As an add-on module, FD$^{2}$ can be seamlessly integrated into decoupled methods and consistently improves performance across diverse fine-grained and general datasets in most settings.
Future work will further improve the performance of FD$^{2}$ on general datasets and extend it to broader distillation paradigms (\eg, ViT-based models) and larger-scale datasets (\eg, ImageNet-1K).

\section*{Acknowledgements}
This study was supported by JSPS KAKENHI Grant Numbers JP24K02942 and JP25K21218, and the Data Intelligence Innovation Base Project under Grant No. FPF10120260003 and the Scientific Embodied Multi-Agent System Project under Grant No. FPF10120250008.

\bibliographystyle{splncs04}
\bibliography{main}

\newpage

\setcounter{section}{0}
\setcounter{subsection}{0}
\setcounter{subsubsection}{0}
\setcounter{figure}{0}
\setcounter{table}{0}
\setcounter{equation}{0}

\section*{Supplementary Material of FD2: A Dedicated
Framework for Fine-Grained Dataset Distillation}

\cref{theoretical analysis} analyzes the effectiveness of FD$^2$; \cref{results} evaluates fine-grained datasets with more distillation methods, and reports efficiency comparisons across other methods; \cref{details} presents implementation details; \cref{ablation} provides additional ablation studies; \cref{visualizations} visualizes distilled samples on CUB-200-2011.

\section{Theoretical Analysis} \label{theoretical analysis}

\subsection{Effectiveness of the Fine-grained Characteristic Constraint}
\textbf{Preliminaries.}
Following the classical discriminative principle, recognition becomes easier when features exhibit smaller intra-class variation and larger inter-class separation~\cite{Fisher1936,CenterLoss}. For sample $i$, let $z_i\in\mathbb{R}^d$ be its feature, $y_i\in\{1,\dots,K\}$ its label, and $\mu_k\in\mathbb{R}^d$ the feature center of class $k$. We define the intra-class deviation
\begin{equation}
r_i=\|z_i-\mu_{y_i}\|_2,
\label{eq:ri}
\end{equation}
and the average center-based margin
\begin{equation}
\bar{\Delta}_i
=
\mathbb{E}_{j\neq y_i}\!\left[
\|z_i-\mu_j\|_2^2-\|z_i-\mu_{y_i}\|_2^2
\right].
\label{eq:avg_center_margin}
\end{equation}

\begin{proposition}\label{prop:1}
Let
\begin{equation}
d_{ij}=\|\mu_j-\mu_{y_i}\|_2,\qquad j\neq y_i.
\label{eq:dij}
\end{equation}
Then the margin in \cref{eq:avg_center_margin} admits the lower bound
\begin{equation}
\bar{\Delta}_i
\ge
\mathbb{E}_{j\neq y_i}\!\left[d_{ij}^2-2r_i d_{ij}\right].
\label{eq:prop1_bound}
\end{equation}
Hence, decreasing $r_i$ or enlarging $\{d_{ij}\}_{j\neq y_i}$ increases $\bar{\Delta}_i$. In particular, if
\begin{equation}
r_i<
\frac{\sum_{j\neq y_i} d_{ij}^2}{2\sum_{j\neq y_i} d_{ij}},
\label{eq:ri_condition}
\end{equation}
then $\bar{\Delta}_i>0$, implying that $z_i$ is, on average, closer to its own class center than to other class centers.

\end{proposition}

\textit{Proof.}
For any $j\neq y_i$,
\begin{align}
\|z_i-\mu_j\|_2^2-\|z_i-\mu_{y_i}\|_2^2
&=
\|\mu_j-\mu_{y_i}\|_2^2
-2\langle z_i-\mu_{y_i},\,\mu_j-\mu_{y_i}\rangle \nonumber\\
&\ge
d_{ij}^2-2\|z_i-\mu_{y_i}\|_2\,\|\mu_j-\mu_{y_i}\|_2 \nonumber\\
&=
d_{ij}^2-2r_i d_{ij},
\label{eq:prop1_pair}
\end{align}
where the inequality follows from Cauchy--Schwarz. Averaging \cref{eq:prop1_pair} over all $j\neq y_i$ gives \cref{eq:prop1_bound}. The condition in \cref{eq:ri_condition} makes the right-hand side positive, hence $\bar{\Delta}_i>0$. \hfill $\square$

In our method, we do not directly optimize the centers $\mu_k$. Instead, as described in Section 3.3, CAL maintains class prototypes $c_k$ as category-level representatives of discriminative features, which helps aggregate discriminative local information into representative vectors~\cite{WSDAN}.

For distilled sample $\tilde{x}_{y,i}$, the proposed fine-grained characteristic constraint is
\begin{equation}
\mathcal{L}_{F}(\tilde{x}_{y,i})
=
\beta \,\ell_2(z_{y,i},c_y)
+
(1-\beta)\left(1-\mathbb{E}_{k\neq y}\big[\ell_2(z_{y,i},c_k)\big]\right),
\qquad \beta\in[0,1],
\label{eq:LF_theory}
\end{equation}
where $\ell_2(u,v)=\|u-v\|_2/(\|u\|_2+\|v\|_2+\varepsilon)$. We introduce the normalized prototype-based discriminative score
\begin{equation}
\mathcal{M}_{y,i}^{(\beta)}
=
(1-\beta)\,\mathbb{E}_{k\neq y}\big[\ell_2(z_{y,i},c_k)\big]
-
\beta\,\ell_2(z_{y,i},c_y),
\label{eq:prototype_score}
\end{equation}
which increases when $\tilde{x}_{y,i}$ becomes closer to $c_y$ and farther from $\{c_k\}_{k\neq y}$ on average.

\begin{corollary}\label{cor:1}
Assume that the maintained prototypes $\{c_k\}_{k=1}^K$ approximate the discriminative class representatives in the learned feature space. Then minimizing \cref{eq:LF_theory} is equivalent to maximizing $\mathcal{M}_{y,i}^{(\beta)}$. Consequently, the proposed fine-grained characteristic constraint improves intra-class compactness and inter-class separability in the normalized prototype space, thereby reducing the recognition difficulty of distilled samples.

\end{corollary}

\textit{Proof.}
Rearranging \cref{eq:LF_theory} gives
\begin{equation}
\mathcal{L}_{F}(\tilde{x}_{y,i})
=
(1-\beta)-\mathcal{M}_{y,i}^{(\beta)}.
\label{eq:LF_score_equiv}
\end{equation}
Therefore,
\begin{equation}
\arg\min \mathcal{L}_{F}(\tilde{x}_{y,i})
=
\arg\max \mathcal{M}_{y,i}^{(\beta)}.
\label{eq:LF_arg_equiv}
\end{equation}
Under the assumption that $\{c_k\}$ approximate discriminative class representatives, maximizing $\mathcal{M}_{y,i}^{(\beta)}$ increases the margin between the target prototype and the other-class prototypes, which is consistent with Proposition~\ref{prop:1}. \hfill $\square$

\subsection{Effectiveness of the Similarity Constraint}

\textbf{Preliminaries.}
Existing decoupled methods generate the same-class samples sequentially. In the $i$-th iteration, a real image from class $y$ is randomly selected from the original dataset as the initialization, and the $i$-th distilled sample is optimized under the same pipeline. The shared sample-wise generation process is written as
\begin{equation}
A_{y,i}=\Psi_y\!\left(x_{y,i}^{(0)}\right),
\label{eq:samplewise_mapping}
\end{equation}
where $A_{y,i}\in\mathbb{R}^{P}$ denotes the final vectorized attention map, $\Psi_y$ denotes the common mapping for class $y$, $x_{y,i}^{(0)}$ denotes the initialization image. If $\Psi_y$ is $L_y$-Lipschitz around the initialization images and the solutions actually reached during optimization, then
\begin{equation}
\|A_{y,i}-A_{y,j}\|_2
\le
L_y\|x_{y,i}^{(0)}-x_{y,j}^{(0)}\|_2.
\qquad i\neq j.
\label{eq:lipschitz_attention}
\end{equation}
For a same-class group of size $N_S$, define
\begin{equation}
\bar{A}_{y}=\frac{1}{N_S}\sum_{i=1}^{N_S}A_{y,i},
\qquad
\Sigma_A^{(y)}
=
\frac{1}{N_S}\sum_{i=1}^{N_S}(A_{y,i}-\bar{A}_{y})(A_{y,i}-\bar{A}_{y})^\top .
\label{eq:attention_covariance}
\end{equation}
Then
\begin{equation}
\operatorname{tr}\!\left(\Sigma_A^{(y)}\right)
=
\frac{1}{2N_S^2}\sum_{i,j=1}^{N_S}\|A_{y,i}-A_{y,j}\|_2^2,
\label{eq:pairwise_identity}
\end{equation}
which characterizes the overall dispersion of the same-class attention maps. Under linear attention pooling $G_y$,
\begin{equation}
h_{y,i}=G_y^\top A_{y,i},
\qquad
\Sigma_h^{(y)}
=
\frac{1}{N_S}\sum_{i=1}^{N_S}(h_{y,i}-\bar{h}_{y})(h_{y,i}-\bar{h}_{y})^\top ,
\label{eq:representation_covariance}
\end{equation}
we have
\begin{equation}
\Sigma_h^{(y)}=G_y^\top \Sigma_A^{(y)} G_y,
\qquad
\sigma_{\min}^2(G_y)\operatorname{tr}\!\left(\Sigma_A^{(y)}\right)
\le
\operatorname{tr}\!\left(\Sigma_h^{(y)}\right)
\le
\|G_y\|_2^2\operatorname{tr}\!\left(\Sigma_A^{(y)}\right).
\label{eq:representation_trace_bound}
\end{equation}

\begin{proposition}\label{prop:2}
Under the shared sample-wise generation process, same-class attention diversity satisfies
\begin{equation}
\operatorname{tr}\!\left(\Sigma_A^{(y)}\right)
\le
\frac{L_y^2}{2N_S^2}\sum_{i,j=1}^{N_S}\|x_{y,i}^{(0)}-x_{y,j}^{(0)}\|_2^2.
\label{eq:attention_diversity_bound}
\end{equation}
Hence, if same-class initialization images are close to each other, the resulting attention maps tend to focus on similar regions. Moreover, by \cref{eq:representation_trace_bound}, the diversity of the resulting representations is also limited.
\end{proposition}

\textit{Proof.}
From \cref{eq:lipschitz_attention},
\begin{equation}
\|A_{y,i}-A_{y,j}\|_2^2
\le
L_y^2\|x_{y,i}^{(0)}-x_{y,j}^{(0)}\|_2^2.
\label{eq:pairwise_lipschitz_sq}
\end{equation}
Substituting \cref{eq:pairwise_lipschitz_sq} into \cref{eq:pairwise_identity} gives \cref{eq:attention_diversity_bound}. Then \cref{eq:representation_trace_bound} shows that smaller $\operatorname{tr}(\Sigma_A^{(y)})$ also restricts $\operatorname{tr}(\Sigma_h^{(y)})$. \hfill $\square$

To counteract this tendency, we use the similarity constraint
\begin{equation}
\mathcal{L}_{S}(\tilde{x}_{y,i})
=
1-\mathbb{E}_{j<i}\!\left[\ell_2\!\left(A_{y,i},A_{y,j}\right)\right],
\qquad 1<i\le N_S,
\label{eq:LS_theory}
\end{equation}
which serves as a penalty against repeated attention. Here, $\ell_2(u,v)=\|u-v\|_2/(\|u\|_2+\|v\|_2+\varepsilon)$ is the symmetrically normalized Euclidean metric. Assume
\begin{equation}
0<\underline{\eta}
\le
\|A_{y,i}\|_2+\|A_{y,j}\|_2+\varepsilon
\le
\bar{\eta},
\qquad \forall\,j<i.
\label{eq:eta_bounds}
\end{equation}
Then
\begin{equation}
\frac{1}{\bar{\eta}}\|A_{y,i}-A_{y,j}\|_2
\le
\ell_2(A_{y,i},A_{y,j})
\le
\frac{1}{\underline{\eta}}\|A_{y,i}-A_{y,j}\|_2.
\label{eq:metric_equivalence}
\end{equation}

\begin{corollary}\label{cor:2}
Minimizing \cref{eq:LS_theory} increases intra-class attention diversity, and thus promotes more diverse discriminative regions and representations.
\end{corollary}

\textit{Proof.}
From \cref{eq:LS_theory},
\begin{equation}
\min \mathcal{L}_{S}(\tilde{x}_{y,i})
\quad\Longleftrightarrow\quad
\max \mathbb{E}_{j<i}\!\left[\ell_2(A_{y,i},A_{y,j})\right].
\label{eq:LS_equiv}
\end{equation}
By \cref{eq:metric_equivalence} and Jensen's inequality,
\begin{equation}
\mathbb{E}_{j<i}\|A_{y,i}-A_{y,j}\|_2^2
\ge
\left(\mathbb{E}_{j<i}\|A_{y,i}-A_{y,j}\|_2\right)^2
\ge
\underline{\eta}^{\,2}
\left(\mathbb{E}_{j<i}\!\left[\ell_2(A_{y,i},A_{y,j})\right]\right)^2.
\label{eq:jensen_bound}
\end{equation}
Moreover, from \cref{eq:pairwise_identity},
\begin{equation}
\operatorname{tr}\!\left(\Sigma_A^{(y)}\right)
=
\frac{1}{2N_S^2}\sum_{p,q=1}^{N_S}\|A_{y,p}-A_{y,q}\|_2^2
\ge
\frac{i-1}{N_S^2}\,\mathbb{E}_{j<i}\|A_{y,i}-A_{y,j}\|_2^2.
\label{eq:trace_lower_bound}
\end{equation}
Combining \cref{eq:jensen_bound,eq:trace_lower_bound} yields
\begin{equation}
\operatorname{tr}\!\left(\Sigma_A^{(y)}\right)
\ge
\frac{i-1}{N_S^2}\,\underline{\eta}^{\,2}
\left(\mathbb{E}_{j<i}\!\left[\ell_2(A_{y,i},A_{y,j})\right]\right)^2.
\label{eq:attention_diversity_lower}
\end{equation}
Thus, minimizing $\mathcal{L}_{S}(\tilde{x}_{y,i})$ enlarges a lower bound on intra-class attention diversity. By \cref{eq:representation_trace_bound}, this further promotes more diverse representations. \hfill $\square$

\section{More Results and Discussions} \label{results}
To further evaluate the effectiveness of FD$^{2}$, we compare it with more SOTA dataset distillation methods. These methods mainly include coreset selection methods (Random, Herding \cite{Herding}, and K-Center \cite{K-Center}), DSA~\cite{DSA} (gradient matching), DM~\cite{DM} (distribution matching), DATM~\cite{DATM}, and TESLA \cite{TESLA} (trajectory matching). We also attempted the generative method Minimax~\cite{Minimax}. However, it failed to converge stably when fine-tuning the pretrained DiT~\cite{DiT} on fine-grained datasets, likely because fine-grained data require modeling subtle inter-class differences, while limited data hinder stable diffusion optimization and sufficient distribution coverage. Therefore, we do not include Minimax in \cref{tab:fg2} and \cref{tab:efficiency}.

All experiments are conducted on CUB-200-2011, FGVC-Aircraft, and Stanford Cars, with all input images resized to 224 $\times$ 224. The coreset methods employ ResNet18 \cite{ResNet} as the student model; the other distillation methods follow their original papers and uniformly use ConvNet \cite{ConvNet} for both distillation and post-evaluation.

\paragraph{Accuracy of Other Methods.}
We further compare other methods on fine-grained datasets. As shown in \cref{tab:fg2}, coreset selection yields relatively low accuracy because fine-grained data exhibit subtle inter-class differences and large intra-class variation. Among these methods, Herding achieves higher accuracy, while Random achieves lower accuracy. Dataset distillation methods perform better, but DSA relies on gradient signals, DM on feature distribution signals, and DATM on weight trajectory signals. These signals are coarse-grained statistics over the whole dataset and thus cannot capture subtle but critical class differences. As the number of classes increases and inter-class distances shrink, such global signals are more easily dominated by shared patterns across classes, making class-specific patterns harder to preserve in distilled data. This limits discriminative representation learning and results in relatively low performance. SRe$^2$L++ shows strong performance, while FD$^2$ further suppresses fine-grained characteristics and similarity, which reduces the difficulty for the student model to learn discriminative representations and ultimately improves accuracy.

\begin{table*}[t]
\centering
\scriptsize
\setlength{\tabcolsep}{2.0pt}
\renewcommand{\arraystretch}{1.0}
\caption{Accuracy of coreset selection methods and other SOTA distillation methods on fine-grained datasets.}

\label{tab:fg2}
\begin{tabular}{c c| c c c | c c c c c}
\toprule
\multirow{2}{*}{Dataset} & \multirow{2}{*}{IPC} &
\multicolumn{3}{c|}{Coreset Selection} &
\multicolumn{5}{c}{Dataset Distillation} \\
\cmidrule(lr){3-5}\cmidrule(lr){6-10}
& & Random & Herding & K-Center & DSA & DM & DATM & TELSA & SRe$^2$L++$_\mathrm{FD^2}$ \\
\midrule

\multirow{3}{*}{CUB-200-2011}
& 1 & 1.5 & 1.5 & 1.4 & 3.6 & 1.8 & 4.3 & 3.1 & 56.4 \\
& 3 & 2.0 & 2.2 & 2.2 & 6.5 & 4.5 & 6.6 & 5.0 & 64.9 \\
& 5 & 2.6 & 3.0 & 2.9 & 10.2 & 7.5 & 8.7 & 5.9 & 67.0 \\
\midrule

\multirow{3}{*}{FGVC-Aircraft}
& 1 & 1.8 & 2.0 & 2.5 & 8.4 & 5.4 & 1.4 & 3.4 & 58.2 \\
& 3 & 3.8 & 4.1 & 4.1 & 11.3 & 11.1 & 10.2 & 6.5 & 76.1 \\
& 5 & 4.4 & 5.0 & 4.8 & 14.9 & 16.0 & 17.7 & 8.7 & 80.0 \\
\midrule

\multirow{3}{*}{Stanford Cars}
& 1 & 1.2 & 1.3 & 1.2 & 7.2 & 6.3 & 6.2 & 2.5 & 64.5 \\
& 3 & 1.3 & 2.0 & 1.8 & 7.4 & 7.1 & 6.3 & 3.9 & 75.2 \\
& 5 & 1.5 & 2.4 & 2.0 & 7.8 & 7.6 & 5.7 & 4.6 & 81.4 \\
\bottomrule
\end{tabular}
\end{table*}

\begin{table}[t]
\centering
\scriptsize
\setlength{\tabcolsep}{3.0pt}
\renewcommand{\arraystretch}{1.0}
\caption{Image optimization time per iteration and peak GPU memory of different methods on CUB-200-2011 at IPC=1.}
\label{tab:efficiency}
\begin{tabular}{c c c c c c c c}
\toprule
 & DSA & DM & DATM & SRe$^2$L++ & SRe$^2$L++$_{\mathrm{FD^{2}}}$ & FADRM+ & FADRM+$_{\mathrm{FD^{2}}}$ \\
\midrule
Time Cost & 3.9 s & 1.4 s & 2.7 s & 64.6 ms & 96.8 ms & 88.9 ms & 134.8 ms \\
Peak GPU & 68.4 GB & 28.9 GB & 50.0 GB & 4.8 GB & 5.3 GB & 12.2 GB & 14.1 GB \\
\bottomrule
\end{tabular}
\end{table}

\paragraph{Distillation Efficiency of Different Methods.}
To evaluate distillation efficiency, we compare the image optimization time per iteration and peak GPU memory of different methods on CUB-200-2011 at IPC=1. Since RDED does not iteratively optimize images via gradient descent but distills samples by one-shot cropping, we do not include it in \cref{tab:efficiency}. As shown in \cref{tab:efficiency}, DSA, which relies on gradient signals, DM, which relies on feature distribution signals, and DATM, which relies on weight trajectory signals, generally require higher time and memory costs. In contrast, decoupled methods based on BN-layer global statistics are more efficient, with lower time and memory overhead, while also achieving better performance. Moreover, compared with SRe$^2$L++ and FADRM+, FD$^2$ introduces only small additional time and memory costs, yet provides clear performance improvements.

\section{Additional Implementation Details} \label{details}
\subsection{Model Pretraining}
As shown in \cref{tab:cal_settings}, we report the accuracy of different models integrated with CAL. On fine-grained datasets, ShuffleNetV2 is designed for lightweight computation and thus cannot provide sufficiently discriminative features for CAL, leading to lower accuracy than the other models. Moreover, a larger $\alpha$ further increases the optimization difficulty; in our experiments, ShuffleNetV2+CAL converges only at $\alpha=0.1$. These results indicate that, on fine-grained datasets, all models except ShuffleNetV2 are well compatible with CAL. In contrast, on the simpler ImageNette and ImageWoof datasets, all models show excellent compatibility with CAL.
\begin{table*}[t]
\centering
\small
\setlength{\tabcolsep}{2pt}
\renewcommand{\arraystretch}{1.0}
\caption{Optimal CAL settings of different models during pretraining and the corresponding accuracy.}
\label{tab:cal_settings}
\resizebox{\linewidth}{!}{%
\begin{tabular}{c c c c c c c}
\toprule
Dataset & Settings & ShuffleNetV2+CAL & MobileNetV2+CAL & DenseNet121+CAL & ResNet18+CAL & ResNet50+CAL \\
\midrule

\multirow{4}{*}{CUB-200-2011}
& Num$_{\text{attention-maps}}$ & 16 & 16 & 16 & 8  & 8  \\
& CAL Ratio $\alpha$            & 0.1 & 0.5 & 0.4 & 0.5 & 0.3 \\
& Backbone Acc                  & 55.1 & 72.3 & 77.4 & 71.6 & 76.1 \\
& CAL Acc                       & 59.7 & 75.3 & 81.2 & 76.1 & 79.6 \\
\midrule

\multirow{4}{*}{FGVC-Aircraft}
& Num$_{\text{attention-maps}}$ & 16 & 16 & 16 & 32 & 32 \\
& CAL Ratio $\alpha$            & 0.1 & 0.2 & 0.4 & 0.3 & 0.4 \\
& Backbone Acc                  & 71.4 & 84.3 & 87.5 & 83.9 & 87.3 \\
& CAL Acc                       & 70.3 & 84.8 & 88.4 & 84.6 & 87.7 \\
\midrule

\multirow{4}{*}{Stanford Cars}
& Num$_{\text{attention-maps}}$ & 8  & 16 & 32 & 8  & 32 \\
& CAL Ratio $\alpha$            & 0.1 & 0.4 & 0.4 & 0.3 & 0.4 \\
& Backbone Acc                  & 71.6 & 86.7 & 88.6 & 85.2 & 89.2 \\
& CAL Acc                       & 76.7 & 88.8 & 91.68 & 88.6 & 91.9 \\
\midrule

\multirow{4}{*}{ImageNette}
& Num$_{\text{attention-maps}}$ & 8  & 16 & 32 & 8  & 32 \\
& CAL Ratio $\alpha$            & 0.3 & 0.1 & 0.1 & 0.3 & 0.2 \\
& Backbone Acc                  & 94.2 & 97.4 & 98.5 & 98.0 & 98.9 \\
& CAL Acc                       & 93.6 & 97.3 & 98.3 & 97.5 & 98.5 \\
\midrule

\multirow{4}{*}{ImageWoof}
& Num$_{\text{attention-maps}}$ & 16  & 32 & 32 & 8  & 16 \\
& CAL Ratio $\alpha$            & 0.4 & 0.1 & 0.1 & 0.3 & 0.2 \\
& Backbone Acc                  & 84.2 & 92.4 & 92.4 & 92.3 & 94.1 \\
& CAL Acc                       & 87.4 & 93.3 & 93.3 & 92.6 & 93.8 \\

\bottomrule
\end{tabular}}
\end{table*}

\subsection{Distillation}
\paragraph{The Models Used in FADRM+ and FADRM+$_{FD^{2}}$.}
We follow the FADRM+ setting and use four models for distillation. However, on fine-grained datasets, due to the poor performance of ShuffleNetV2+CAL, we use MobileNetV2, DenseNet121, ResNet18, and ResNet50 in both FADRM+ and FADRM+$_\mathrm{FD^{2}}$. For general datasets, we retain the original FADRM+ setting and use ShuffleNetV2, MobileNetV2, DenseNet121, and ResNet18 for distillation.
\paragraph{Image Initialization.}
Following the settings of SRe$^2$++ and FADRM+, SRe$^2$++$_\mathrm{FD^{2}}$ and FADRM+$_\mathrm{FD^{2}}$ also adopt $2\times2$ and $1\times1$ image initialization, respectively.
\paragraph{Group-wise Distillation with Group Size $N_S$.}
We optimize the IPC distilled images of each class in a group-wise manner, where each group contains at most $N_S$ images, and each iteration corresponds to one image. Therefore, the total number of groups for each class is
\[G=\left\lceil \frac{\mathrm{IPC}}{N_S} \right\rceil .\]
The first $G-1$ groups all run for $N_S$ iterations, while the last group runs for
\[N_{\mathrm{last}}=\mathrm{IPC}-(G-1)N_S\]
iterations. This strategy prevents the similarity constraint from introducing weakly discriminative regions into the distilled images while maintaining distillation efficiency.

\section{Addition ablation studies} \label{ablation}
\paragraph{Effect of the Group Size $N_S$ in the Similarity Constraint on the Distillation Time.}
We further compare the distillation time under different $N_S$ for the similarity constraint. Each group runs for $N_S$ iterations, and the reported time is the mean per-iteration cost within the group. As shown in \cref{tab:ns_time}, the time overhead increases with $N_S$ and reaches the highest value (1739.4 s) at $N_S=5$. This is because more previously distilled samples are involved in the similarity constraint as the iterations proceed. Considering accuracy, $N_S=4$ is preferred.

\paragraph{Effect of the FD$^2$-specific constraint term $\mathcal{L}_{\mathrm{cls}}$.}
Since FD$^2$ adjusts the class level semantic supervision during distillation by using both the backbone classifier and the CAL classifier, we further conduct an ablation study on $\mathcal{L}_{\mathrm{cls}}$. As shown in \cref{tab:l_cls}, removing $\mathcal{L}_{\mathrm{cls}}$ causes a clear performance drop, because the distilled samples lose class semantic supervision, which is crucial for recognition. Therefore, introducing this constraint term is necessary.

\noindent
\begin{minipage}[t]{0.49\textwidth}\vspace{0pt}
  \centering
  \small
  \captionsetup{type=table,hypcap=false,skip=2pt,justification=justified,singlelinecheck=false,width=\linewidth}
  \captionof{table}{Effect of $N_S$ in the similarity constraint on distillation time.}
  \label{tab:ns_time}
  \setlength{\tabcolsep}{2pt}
  \renewcommand{\arraystretch}{0.9}
  \begin{tabular}{c c c}
    \toprule
    Method & $N_S$ & Time Cost (s) \\
    \midrule
    \multirow{4}{*}{SRe$^2$L++$_{\mathrm{FD}^{\scriptscriptstyle 2}}$}
      & 2 & 1708.3 \\
      & 3 & 1709.0 \\
      & 4 & 1713.1 \\
      & 5 & 1739.4 \\
    \bottomrule
  \end{tabular}
\end{minipage}\hspace{0.02\textwidth}%
\begin{minipage}[t]{0.49\textwidth}\vspace{0pt}
  \centering
  \scriptsize
  \captionsetup{type=table,hypcap=false,skip=2pt,justification=justified,singlelinecheck=false,width=\linewidth}
  \captionof{table}{Effect of the FD$^2$-specific constraint term $\mathcal{L}_{\mathrm{cls}}$.}
  \label{tab:l_cls}
  \setlength{\tabcolsep}{0.5pt}
  \renewcommand{\arraystretch}{1.5}
  \resizebox{0.7\linewidth}{!}{%
  \begin{tabular}{c c c c}
    \toprule
    Method & IPC & $\mathcal{L}_{\mathrm{cls}}$ & Acc \\
    \midrule
    \multirow{2}{*}{SRe$^2$L++$_{\mathrm{FD}^{\scriptscriptstyle 2}}$}
      & \multirow{2}{*}{3} & w/  & 60.0 \\
      &                    & w/o & 57.2 \\
    \bottomrule
  \end{tabular}}
\end{minipage}

\section{More Visualizations} \label{visualizations}
At IPC=1, the fine-grained distilled samples obtained by FD$^2$ are shown in \cref{fig:cub_sre2l_fd2_ipc1_page1,fig:cub_sre2l_fd2_ipc1_page2,fig:cub_fadrm_fd2_ipc1_page1,fig:cub_fadrm_fd2_ipc1_page2} for CUB-200-2011, \cref{fig:aircraft_sre2l_fd2_ipc1_page1,fig:aircraft_fadrm_fd2_ipc1_page1} for FGVC-Aircraft, and \cref{fig:car_sre2l_fd2_ipc1_page1,fig:car_sre2l_fd2_ipc1_page2,fig:car_fadrm_fd2_ipc1_page1,fig:car_fadrm_fd2_ipc1_page2} for Stanford Cars.

\FloatBarrier
\clearpage

\begin{figure}[t]
  \centering
  \includegraphics[width=\linewidth]{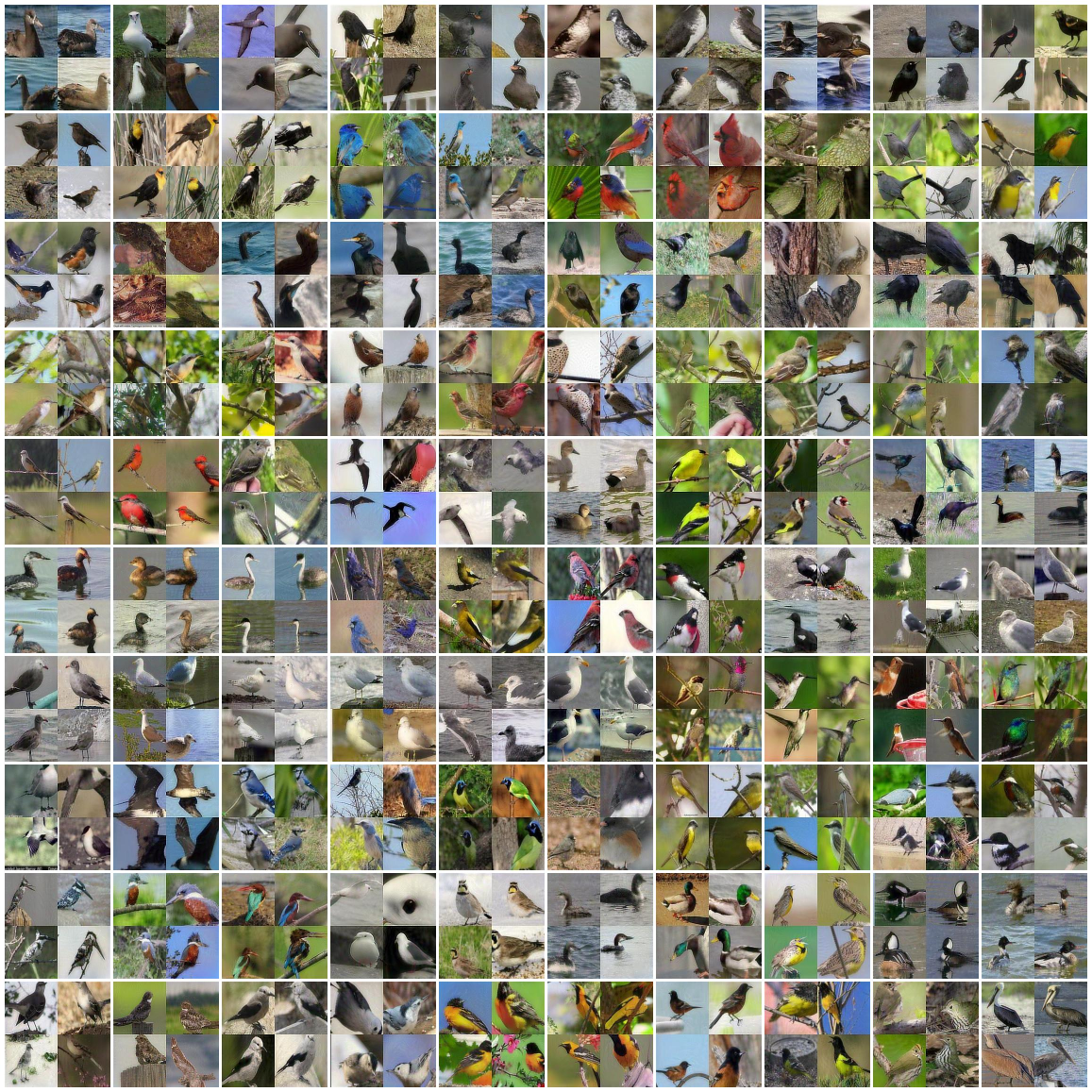}
  \caption{Visualization of distilled samples from the first 100 classes on CUB-200-2011 obtained by SRe$^2$++$_\mathrm{FD^{2}}$ at IPC$=1$.}
  \label{fig:cub_sre2l_fd2_ipc1_page1}
\end{figure}

\begin{figure}[p]
  \centering
  \includegraphics[width=\linewidth]{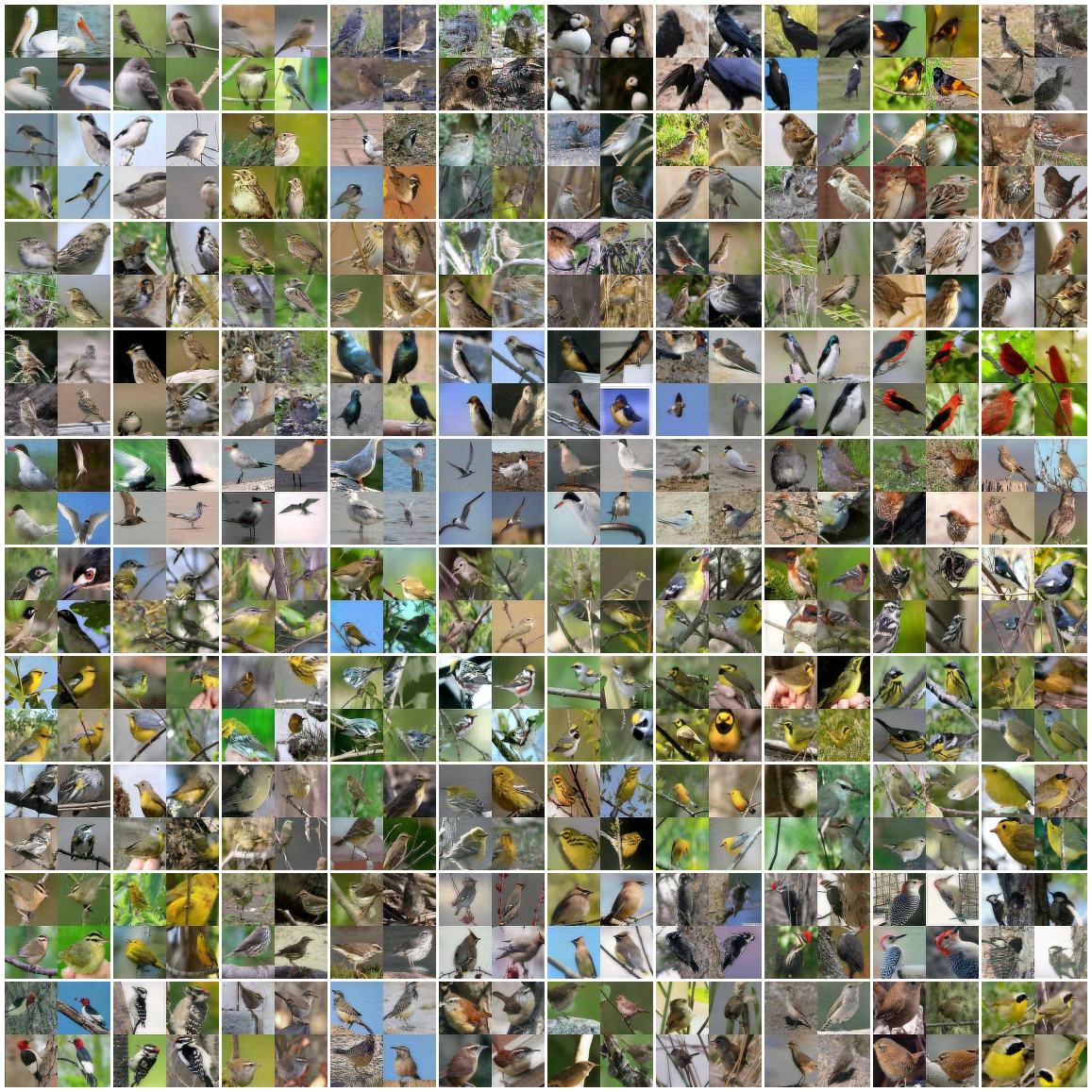}
  \caption{Visualization of distilled samples from the last 100 classes on CUB-200-2011 obtained by SRe$^2$++$_\mathrm{FD^{2}}$ at IPC$=1$.}
  \label{fig:cub_sre2l_fd2_ipc1_page2}
\end{figure}

\begin{figure}[p]
  \centering
  \includegraphics[width=\linewidth]{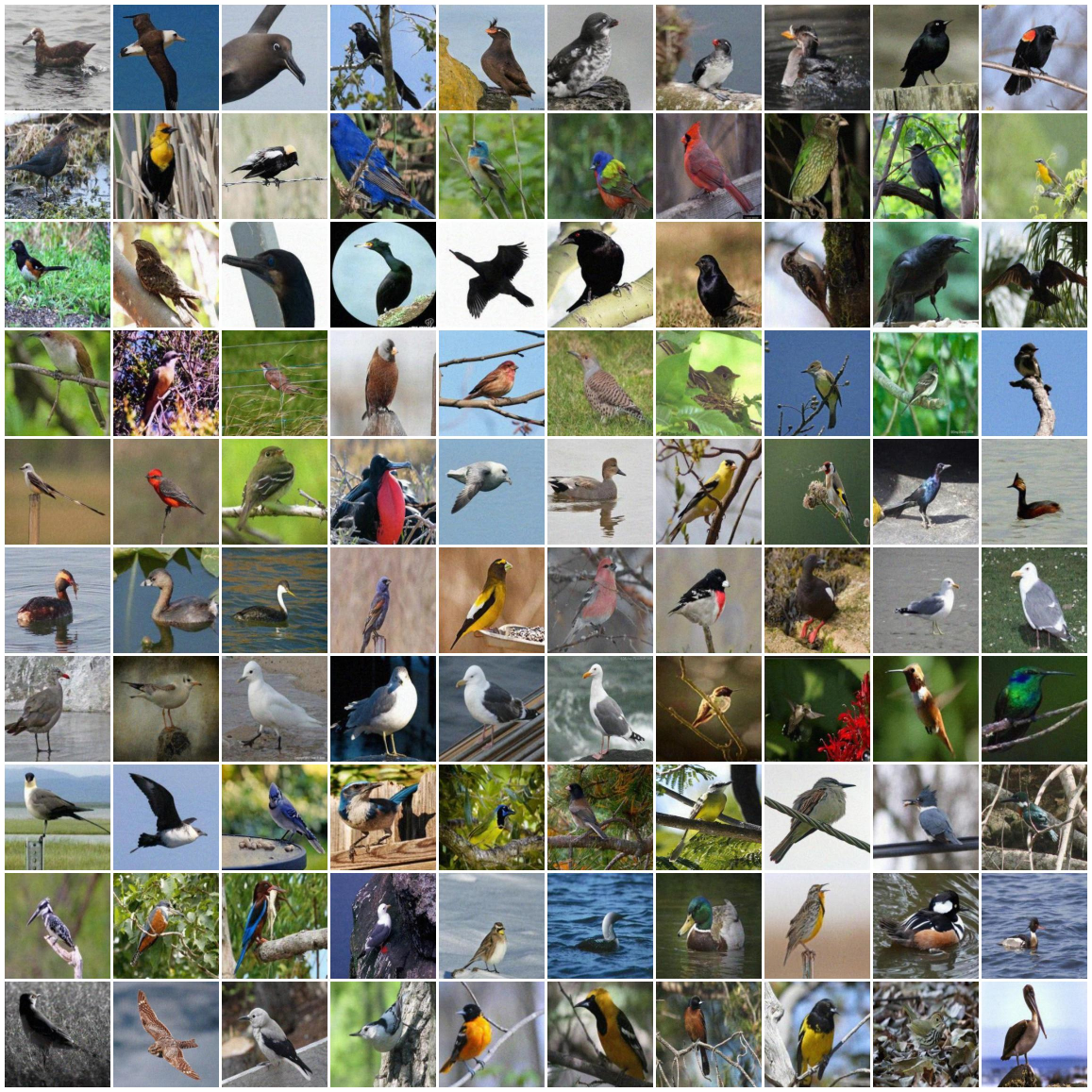}
  \caption{Visualization of distilled samples from the first 100 classes on CUB-200-2011 obtained by FADRM+$_\mathrm{FD^{2}}$ at IPC$=1$.}
  \label{fig:cub_fadrm_fd2_ipc1_page1}
\end{figure}

\begin{figure}[p]
  \centering
  \includegraphics[width=\linewidth]{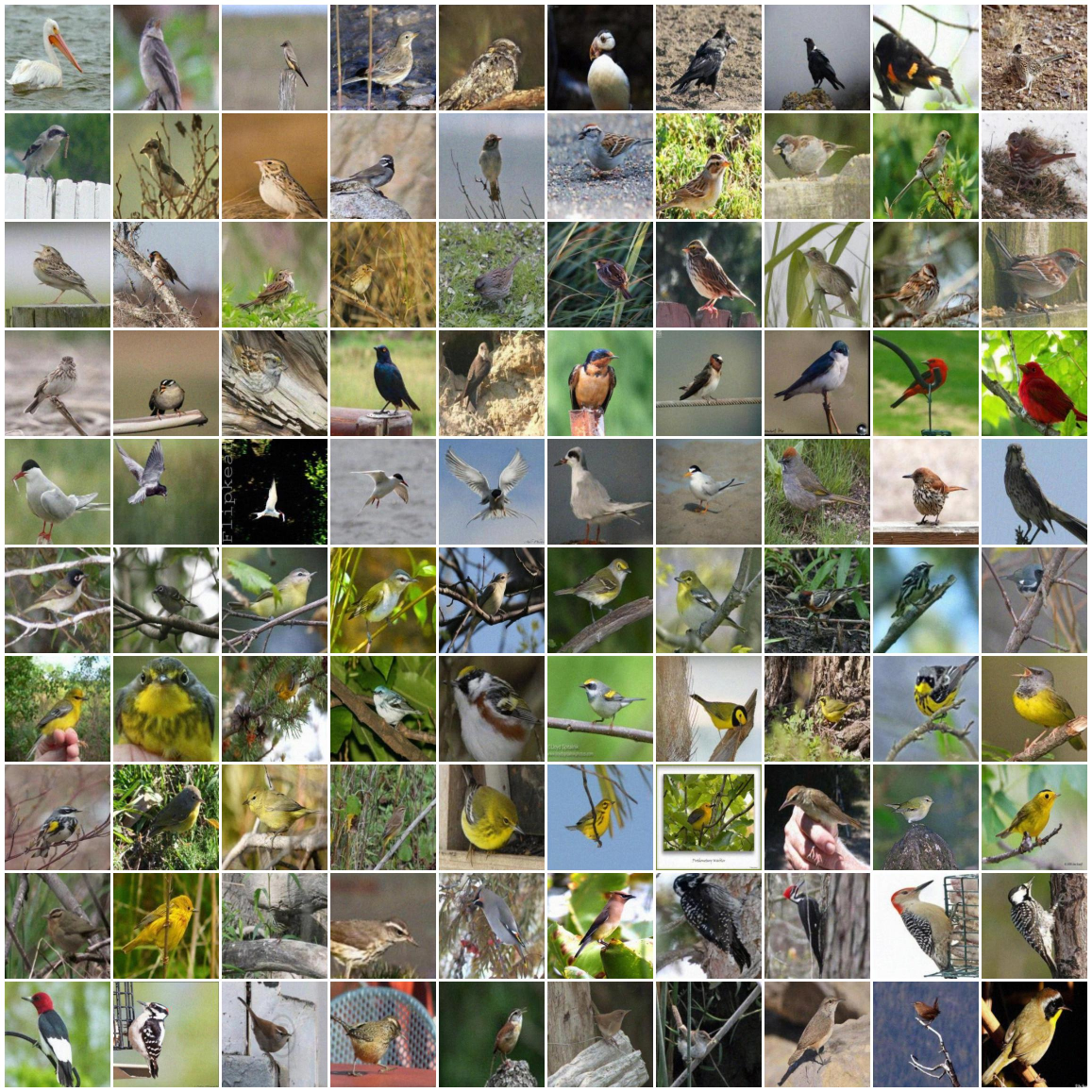}
  \caption{Visualization of distilled samples from the last 100 classes on CUB-200-2011 obtained by FADRM+$_\mathrm{FD^{2}}$ at IPC$=1$.}
  \label{fig:cub_fadrm_fd2_ipc1_page2}
\end{figure}

\begin{figure}[p]
  \centering
  \includegraphics[width=\linewidth]{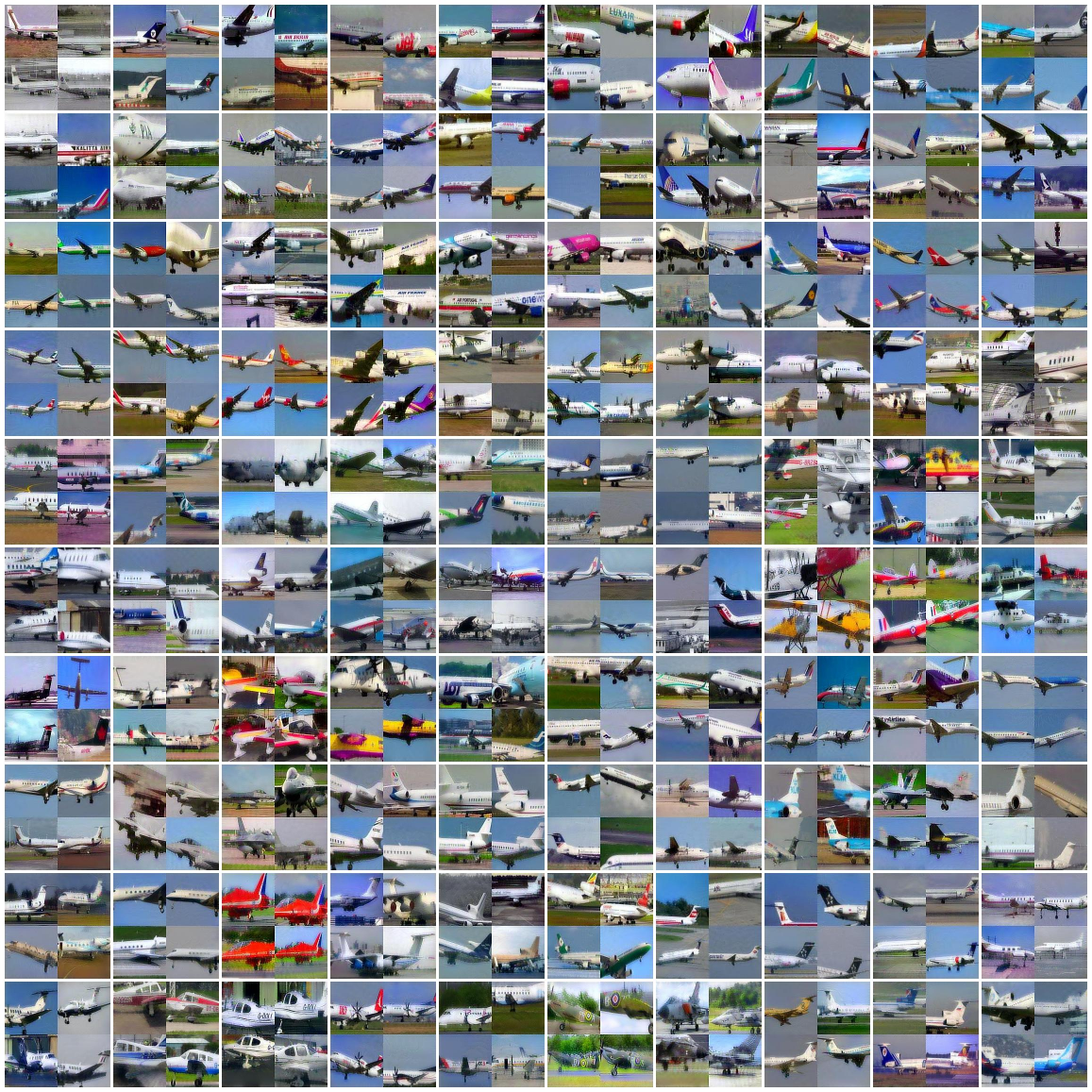}
  \caption{Visualization of distilled samples on FGVC-Aircraft obtained by SRe$^2$++$_\mathrm{FD^{2}}$ at IPC$=1$.}
  \label{fig:aircraft_sre2l_fd2_ipc1_page1}
\end{figure}

\begin{figure}[p]
  \centering
  \includegraphics[width=\linewidth]{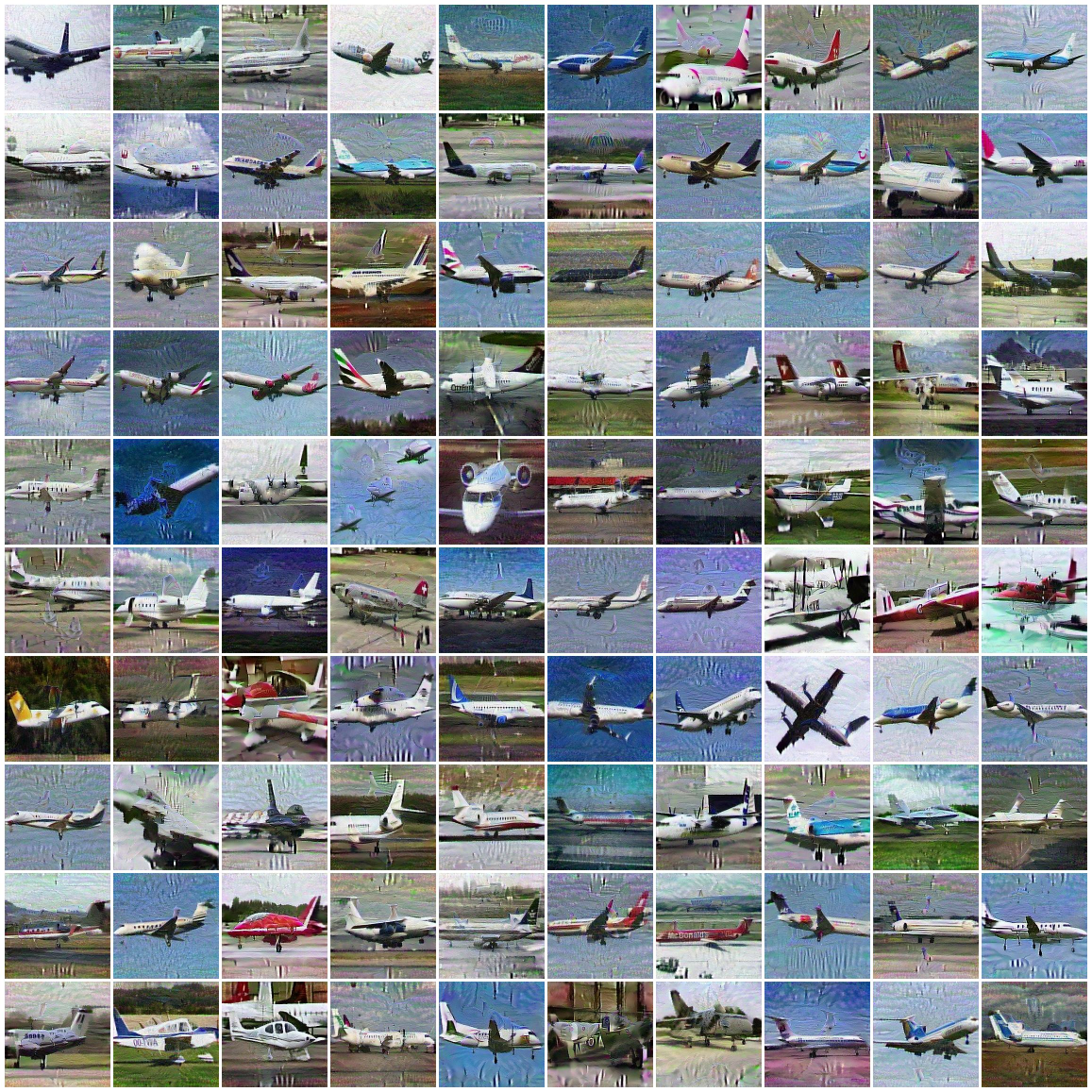}
  \caption{Visualization of distilled samples on FGVC-Aircraft obtained by FADRM+$_\mathrm{FD^{2}}$ at IPC$=1$.}
  \label{fig:aircraft_fadrm_fd2_ipc1_page1}
\end{figure}

\begin{figure}[p]
  \centering
  \includegraphics[width=\linewidth]{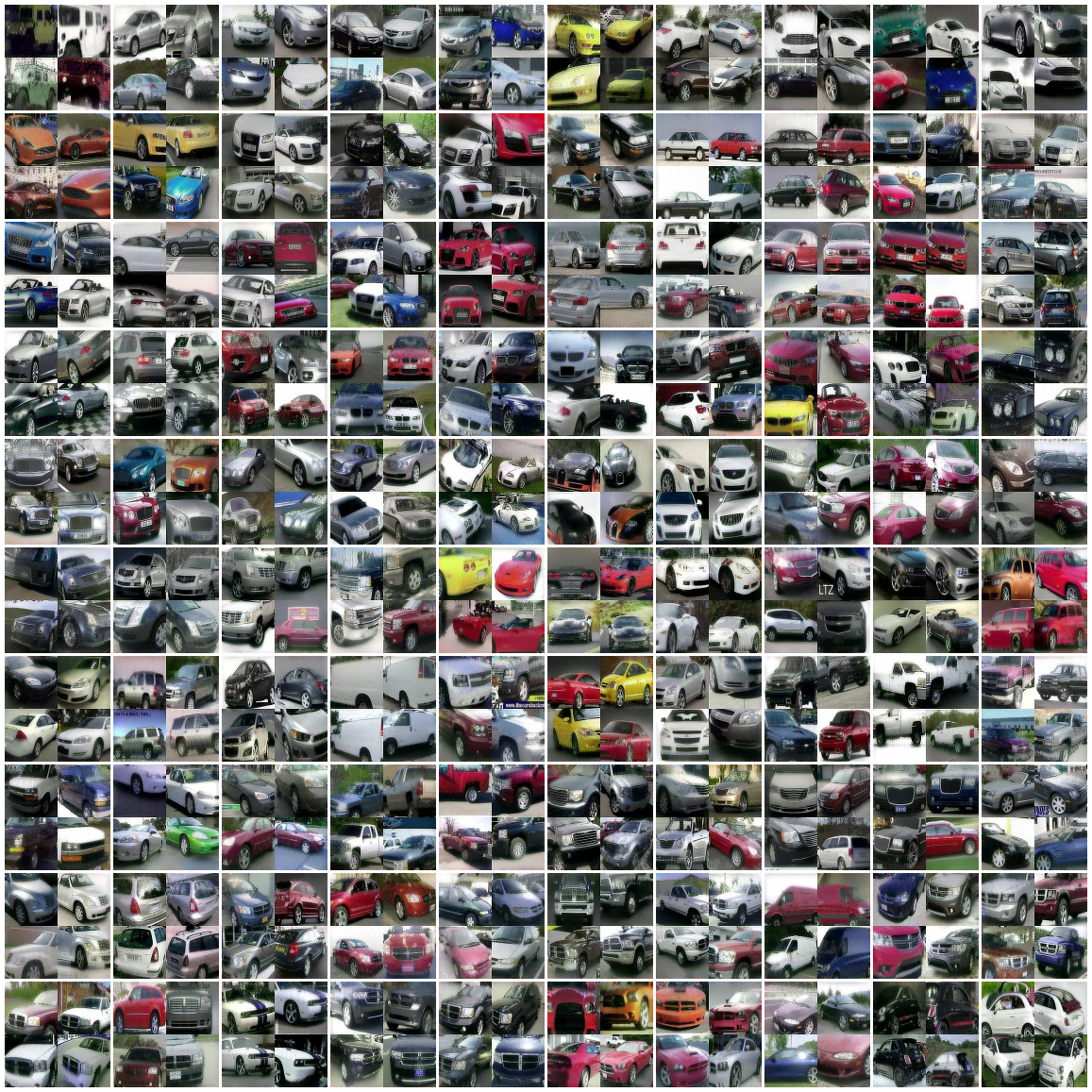}
  \caption{Visualization of distilled samples from the first 100 classes on Standford Cars obtained by SRe$^2$++$_\mathrm{FD^{2}}$ at IPC$=1$.}
  \label{fig:car_sre2l_fd2_ipc1_page1}
\end{figure}

\begin{figure}[p]
  \centering
  \includegraphics[width=\linewidth]{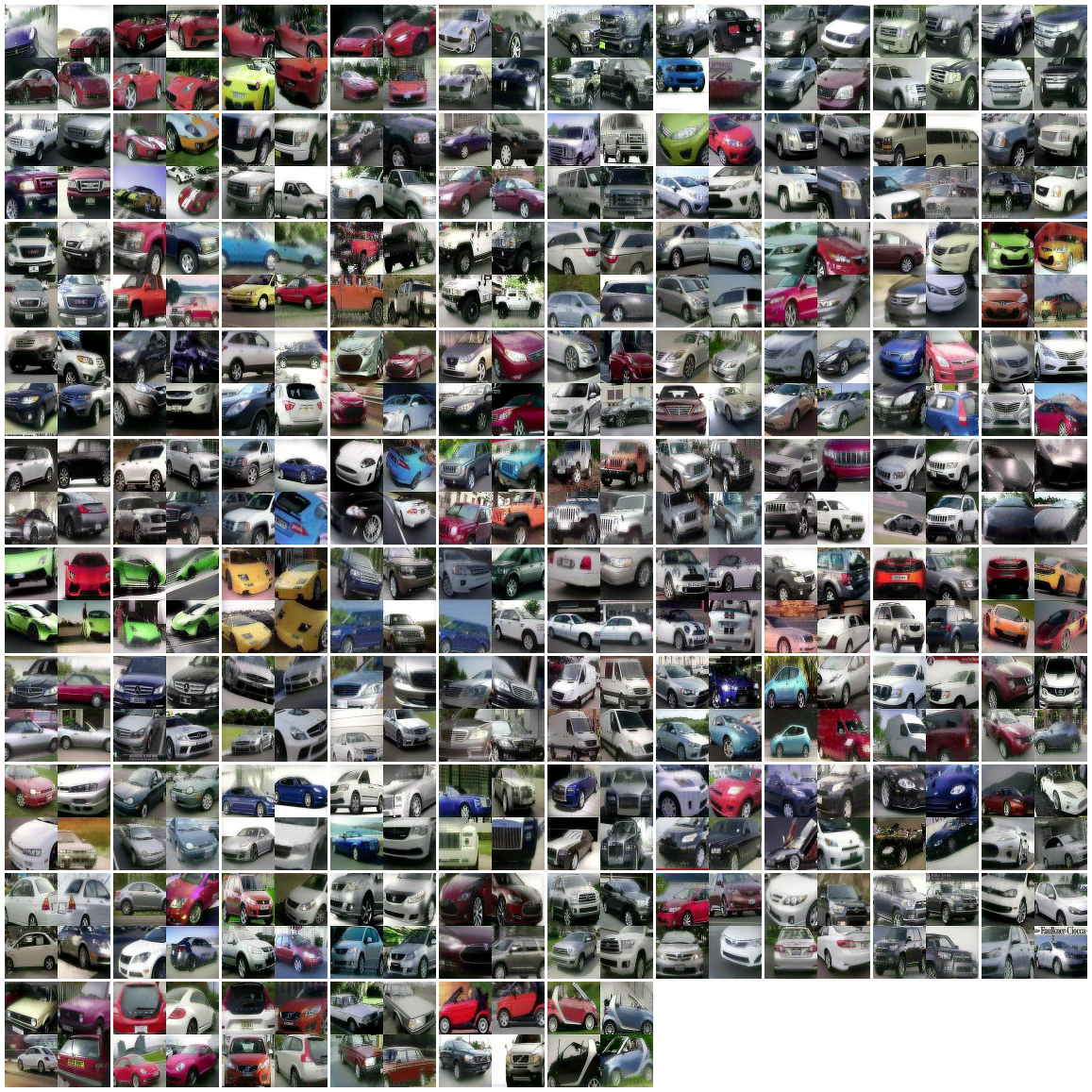}
  \caption{Visualization of distilled samples from the last 96 classes on Standford Cars obtained by SRe$^2$++$_\mathrm{FD^{2}}$ at IPC$=1$.}
  \label{fig:car_sre2l_fd2_ipc1_page2}
\end{figure}

\begin{figure}[p]
  \centering
  \includegraphics[width=\linewidth]{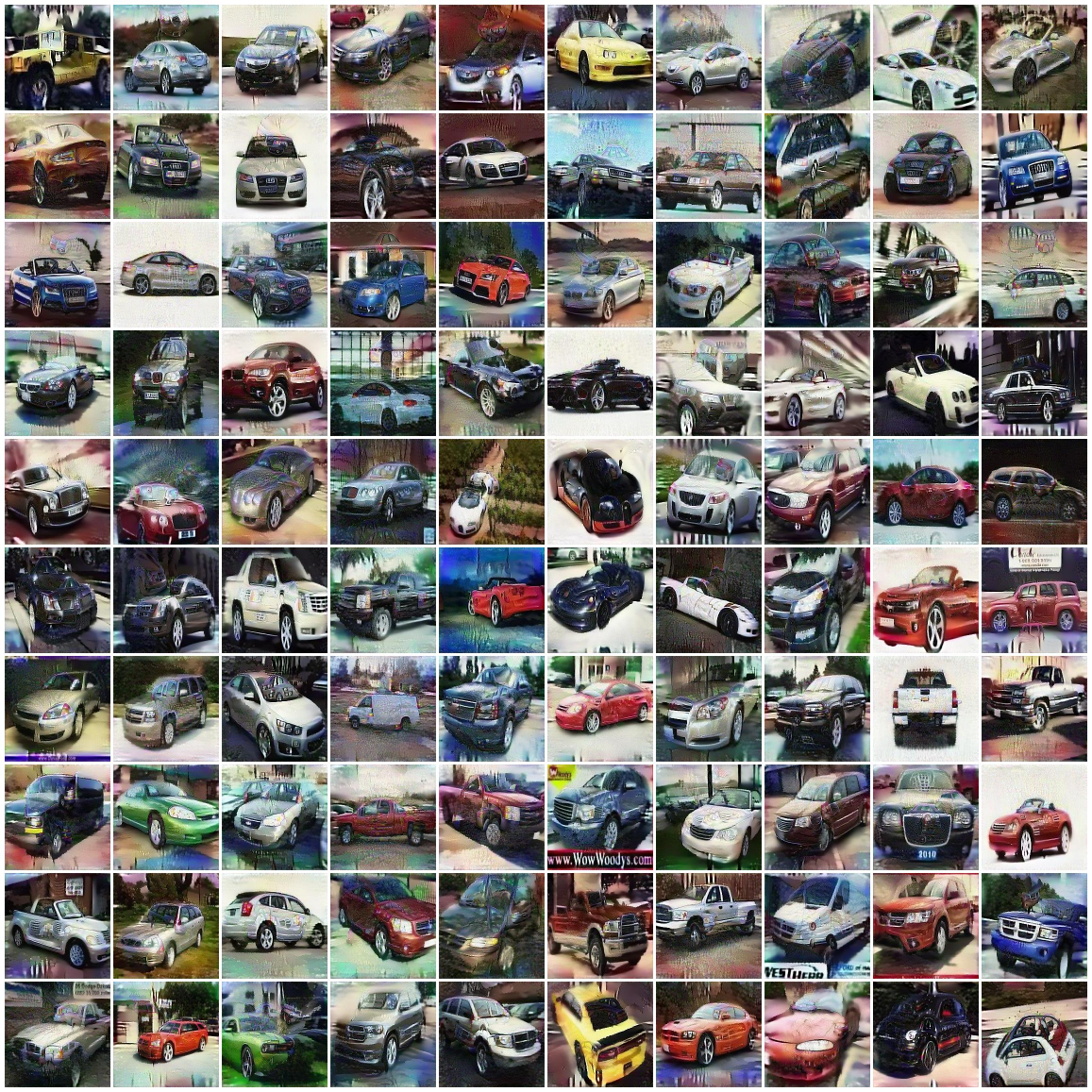}
  \caption{Visualization of distilled samples from the first 100 classes on Standford Cars obtained by FADRM+$_\mathrm{FD^{2}}$ at IPC$=1$.}
  \label{fig:car_fadrm_fd2_ipc1_page1}
\end{figure}

\begin{figure}[p]
  \centering
  \includegraphics[width=\linewidth]{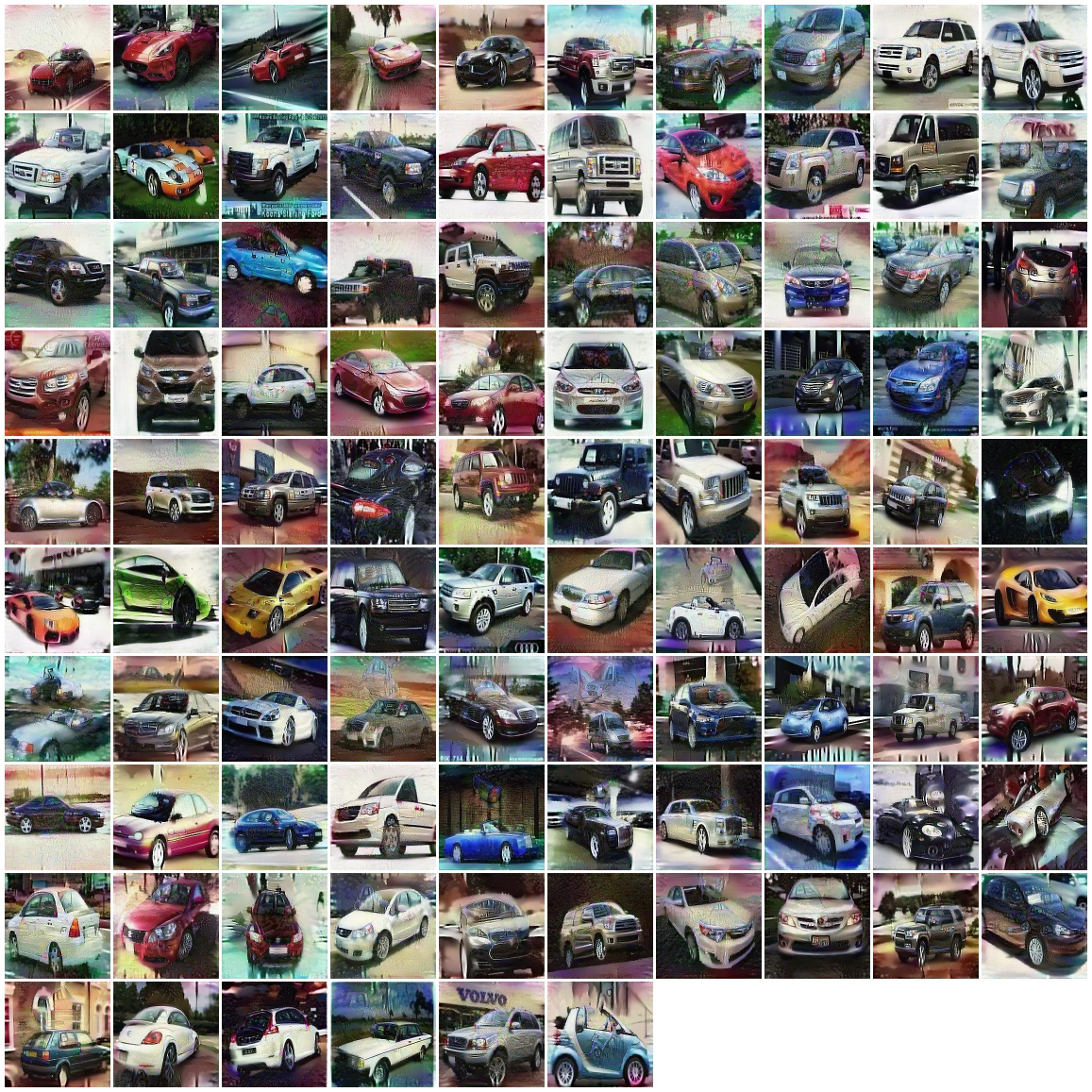}
  \caption{Visualization of distilled samples from the last 96 classes on Standford Cars obtained by FADRM+$_\mathrm{FD^{2}}$ at IPC$=1$.}
  \label{fig:car_fadrm_fd2_ipc1_page2}
\end{figure}

\end{document}